\documentclass{article}
\usepackage[utf8]{inputenc}
\usepackage[T1]{fontenc}
\usepackage{times}
\usepackage{helvet}
\usepackage{courier}

\usepackage{amsmath}
\usepackage{amssymb}
\usepackage{amsfonts}

\usepackage{booktabs}
\usepackage{multirow}
\usepackage{makecell}
\usepackage{array}
\usepackage{tabularx}

\usepackage{graphicx}
\usepackage{xcolor}
\usepackage{caption}

\usepackage{algorithm}
\usepackage{algpseudocode}
\usepackage{enumitem}
\usepackage{tcolorbox}
\usepackage{float}
\usepackage{nicefrac}
\usepackage{microtype}

\usepackage{url}
\usepackage[colorlinks=true, linkcolor=blue, citecolor=blue, urlcolor=blue]{hyperref}

\usepackage{tikz}
\usepackage{booktabs}    
\usepackage{multirow}    
\usepackage{amssymb}     
\usepackage{subcaption}  
\usepackage{graphicx}    
\usepackage{siunitx}     
\usepackage{caption}     
\usepackage[numbers,square]{natbib}

\usepackage{arxiv}

\usepackage{amsthm}

\newtheorem{theorem}{Theorem}[section]
\newtheorem{lemma}[theorem]{Lemma}
\newtheorem{proposition}[theorem]{Proposition}

\title{SwiftGS: Episodic Priors for Immediate Satellite Surface Recovery}

\author{
    Rong Fu \\
    Independent Researcher \\
    Corresponding author \and
    Jiekai Wu \\
    Independent Researcher \and
    Xiaowen Ma \\
    Independent Researcher \and
    Shiyin Lin \\
    Independent Researcher \and
    Kangan Qian \\
    Independent Researcher \and
    Chuang Liu \\
    Independent Researcher \and
    Simon James Fong \\
    Independent Researcher
}

\renewcommand{\headeright}{}
\renewcommand{\undertitle}{}
\renewcommand{\shorttitle}{SwiftGS}

\hypersetup{
    pdftitle={SwiftGS: Episodic Priors for Immediate Satellite Surface Recovery},
    pdfsubject={cs.CV, cs.LG},
    pdfauthor={Rong Fu, Jiekai Wu, Xiaowen Ma, Shiyin Lin, Kangan Qian, Chuang Liu, Simon James Fong},
    pdfkeywords={Gaussian splatting, Hybrid representation, Meta-learning, Zero-shot 3D reconstruction, Shadow-aware},
    pdfstartview={FitH},
    colorlinks=true,
    linkcolor=red,
    citecolor=green,
    filecolor=magenta,
    urlcolor=cyan
}

\begin{document}
\maketitle

\begin{abstract}
Rapid, large-scale 3D reconstruction from multi-date satellite imagery is vital for environmental monitoring, urban planning, and disaster response, yet remains difficult due to illumination changes, sensor heterogeneity, and the cost of per-scene optimization. We introduce \textbf{SwiftGS}, a meta-learned system that reconstructs 3D surfaces in a single forward pass by predicting geometry–radiation-decoupled Gaussian primitives together with a lightweight SDF, replacing expensive per-scene fitting with episodic training that captures transferable priors. The model couples a differentiable physics graph for projection, illumination, and sensor response with spatial gating that blends sparse Gaussian detail and global SDF structure, and incorporates semantic–geometric fusion, conditional lightweight task heads, and multi-view supervision from a frozen geometric teacher under an uncertainty-aware multi-task loss. At inference, SwiftGS operates zero-shot with optional compact calibration and achieves accurate DSM reconstruction and view-consistent rendering at significantly reduced computational cost, with ablations highlighting the benefits of the hybrid representation, physics-aware rendering, and episodic meta-training.
\end{abstract}

\keywords{Gaussian splatting, Hybrid representation, Meta-learning, Zero-shot 3D reconstruction, Shadow-aware}

\section{Introduction}
\label{sec:intro}

High-resolution satellite imagery underpins applications in environmental monitoring, urban planning, and large-scale mapping. Recovering consistent three-dimensional geometry and photometric properties from multi-temporal collections remains challenging, as assumptions in traditional stereo and multi-view stereo pipelines often break under heterogeneous, multi-date archives with varying sensor geometries, illumination, and surface reflectance \cite{zhao2023review,wang2021machine}. Neural rendering approaches, including NeRF variants, provide high-fidelity view synthesis and geometry but typically require costly per-scene optimization, limiting scalability across large repositories \cite{zhang2024fvmd,xiao2025neural}. Explicit primitive-based methods such as 3D Gaussian splatting offer fast training and rendering \cite{dalal2024gaussian}, yet many Earth-observation systems still depend on per-scene fitting, dense coverage, or precise RPC metadata to reach state-of-the-art DSM accuracy, which hinders zero-shot deployment to new regions, sensors, or sparse inputs. Satellite-specific difficulties, including strong view-dependent shadows, radiometric mismatches across dates, transient objects, and occasional RPC errors, further complicate reconstruction \cite{zhao2023review,zhang2023rpcprf,behari2024sundial}. Addressing these challenges requires representations and training protocols that encode robust, transferable priors while supporting compact, physically interpretable local adaptations.

Meta-learning and episodic training provide a principled path to transfer by optimizing across diverse scene distributions so that models acquire priors that generalize to novel episodes and obviate per-scene optimization at inference \cite{nag2023reconstruction,guan2024efficient}. For satellite imagery, effective meta-trained systems require a representation that preserves fine local geometry and globally coherent topology by balancing Gaussian primitives, which capture high-frequency and view-dependent structure but lack explicit topological constraints, with implicit signed-distance fields, which ensure watertight continuity yet may over-smooth sharp discontinuities. The model should also include a physically grounded image formation component that accounts for projection geometry, illumination, atmospheric attenuation, and sensor-specific radiometry while remaining robust to noisy or incomplete RPC metadata. The predictor and training protocol should deliver immediate and accurate zero-shot predictions and support compact, physically interpretable per-scene calibration to accommodate domain shift.

Our contributions are as follows. \textbf{Firstly}, we propose a hybrid scene representation that combines geometry and radiance decoupled Gaussian primitives with a compact implicit signed-distance field via learned spatial gating, where Gaussians model high-frequency geometry and view-dependent appearance through separate covariance parameterizations and the SDF preserves global continuity and topology while enabling explicit handling of shadows and illumination; in parallel, a differentiable physics-based rendering graph integrates projection models, atmospheric effects, and sensor response to support shadow-aware synthesis and robust behavior under RPC uncertainty. \textbf{Secondly}, we design a predictor with bidirectional semantic and geometric fusion and dynamic expert routing so that computation adapts to scene content, and we train it with an episodic meta-learning protocol augmented by auxiliary multi-view stereo guidance to achieve strong zero-shot generalization without per-scene optimization. \textbf{Finally}, we enable efficient scene-specific refinement through a compact low-dimensional calibration vector and demonstrate state-of-the-art zero-shot DSM accuracy at substantially reduced inference cost, with ablations confirming the central roles of the hybrid representation, physics-aware rendering, and meta-training, and with extensions to continual learning via incremental Gaussian memory and information-theoretic active view selection for scalable operation on streaming data.

\section{Related Work}
\label{sec:related}

\subsection{Neural rendering for satellite photogrammetry}
Satellite imagery presents unique reconstruction challenges distinct from conventional photography, including RPC camera models, extreme illumination variations, and multi-temporal radiometric inconsistencies. Early neural radiance field adaptations for Earth observation focused on incorporating rational polynomial camera geometries and handling transient objects \cite{mari2022sat,mari2023multi,mathihalli2024dreamsat}. Subsequent developments addressed shadow modeling and lighting decomposition for accurate elevation recovery in shaded regions \cite{derksen2021shadow,behari2024sundial}. Recent Gaussian splatting pipelines have demonstrated improved efficiency for satellite applications, though most retain per-scene optimization requirements \cite{aira2025gaussian,bai2025satgs,huang2025skysplat,nguyen2024satsplatyolo}. SAT-NGP introduced fast relightable reconstruction using neural graphics primitives \cite{billouard2024sat}, while tile-based frameworks enable scaling to global coverage \cite{billouard2025tile}. These methods establish foundations for our zero-shot generalization objective, yet they fundamentally rely on scene-specific fitting or dense view coverage. SwiftGS advances beyond these limitations through meta-learned priors that eliminate per-scene optimization.

\subsection{Geometric calibration and physical illumination modeling}
Accurate satellite reconstruction demands explicit handling of RPC metadata uncertainties and physically-grounded illumination representation. Practical RPC evaluation strategies employ cached-basis decompositions and low-rank approximations to manage computational complexity \cite{gao2023general,xu2024robustmvs,lee2025visibility}. Self-calibration techniques address geometric inaccuracies through learned correction modules \cite{zhang2023meta}. Shadow detection benchmarks emphasize geometry-aware supervision for robust performance across diverse solar conditions \cite{masquil2025s,du2025gs}. Transient object robustness and appearance stabilization rely on uncertainty-weighted losses and consistency regularization \cite{chougule2025novel,fu2025robustsplat}. Our differentiable physics graph extends these principles through fully differentiable projection, atmospheric attenuation, and sensor response modeling, enabling end-to-end optimization while maintaining interpretability.

\subsection{Meta-learning and transferable priors for reconstruction}
Episodic training paradigms offer principled mechanisms for encoding generalizable knowledge across scene distributions. Remote sensing applications demonstrate fast adaptation for hyperspectral classification and band selection through cross-domain meta-learning \cite{hu2023cross,feng2022mr,li2025contrastive,huang2025hzscm}. Vision research establishes that episodic optimization produces priors generalizing across unseen scenes \cite{soh2020meta,zhang2023meta,sun2024metacap,cho2021camera}. Recent meta-learned radiance and splatting models specifically target out-of-distribution robustness with lightweight fine-tuning capabilities \cite{he2024metags,liao2025zero}. SwiftGS distinguishes itself through hybrid primitive prediction combining Gaussian splatting with implicit surfaces, coupled with extremely compact calibration vectors that preserve zero-shot inference efficiency.

\subsection{Geometric supervision and scalable system design}
Multi-view stereo provides complementary geometric guidance for radiance-based reconstruction. Sparse-to-dense MVS networks improve height estimation under limited baselines through adaptive multi-scale processing \cite{yu2020fast,wang2023adaptive,xu2024robustmvs,lee2025visibility}. Coarse-to-fine cost volumes with importance sampling reduce memory footprints while preserving accuracy \cite{xiao2025mcgs,wu20243d}. Large-area deployment necessitates tiling strategies, distributed processing, and adaptive density control \cite{chen2024dogs,wu20243d,rai2025uvgs}. Conservative growth mechanisms and streaming architectures support real-time applications \cite{franke2025vr,held20253d,gong2025adaptive}. Our framework integrates an auxiliary MVS teacher for training-time geometric priors while maintaining single-pass inference efficiency, alongside incremental memory management for continual learning scenarios.

\subsection{Positioning relative to contemporary methods}
Recent convergence on compact primitives, robust supervision, and meta-learning motivates our unified approach. EO-NeRF achieves accurate DSM recovery through per-scene optimization with explicit shadow modeling \cite{mari2023multi}, yet requires hours of training per scene. EOGS accelerates through Gaussian splatting but retains fitting dependencies \cite{aira2025gaussian}. SkySplat explores generalizable satellite splatting using sparse views \cite{huang2025skysplat}, though without explicit topological constraints or physics-aware rendering. SwiftGS uniquely combines single-forward prediction, hybrid Gaussian-SDF representation with learned gating, differentiable physics graphs, and minimal per-scene calibration, achieving superior zero-shot accuracy with orders of magnitude faster inference.

\begin{figure*}[t]
  \centering
  \includegraphics[width=0.9\textwidth]{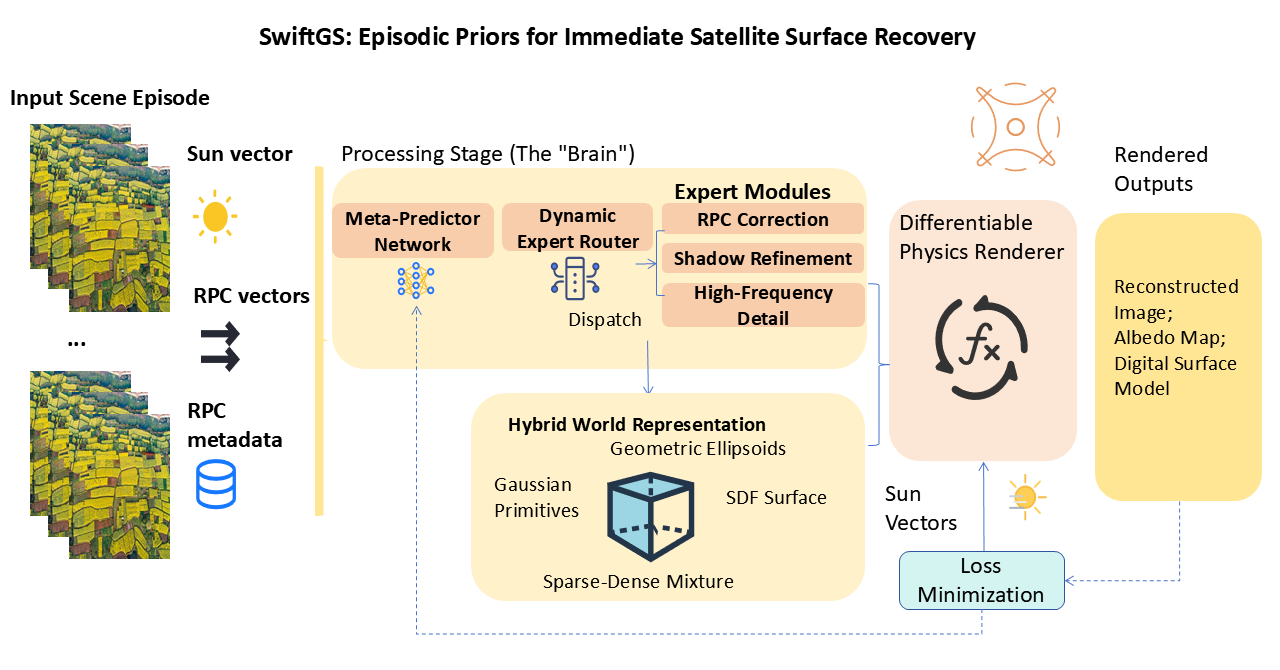}
 \caption{Overview of the \textbf{SwiftGS} architecture for efficient, zero-shot satellite surface reconstruction. The pipeline begins with \textbf{Multi-View Encoding}, where per-view features and a global scene latent are extracted. A \textbf{Hybrid Decoder} produces a compact Gaussian set \(\Gamma\), an implicit SDF \(S_{\psi}\), and spatial gates that blend sparse and dense components. \textbf{Lightweight Task-Specific Heads} optionally refine geometry or appearance. A \textbf{Differentiable Physics Graph} renders elevation and albedo using camera geometry and illumination, supported by distilled geometric cues during training. An episodic meta-training scheme learns shared parameters \(\Phi\) with a small per-scene calibration vector \(\theta\). At inference, SwiftGS runs zero-shot and outputs a DSM, consistent renderings, and a compact Gaussian memory for incremental reconstruction.}
  \label{fig:swiftgs_framework}
\end{figure*}

\section{Methodology}
\label{sec:method}

This section formalizes the episodic reconstruction problem and introduces the SwiftGS design in two parts: the core architecture containing the main contributions, and the auxiliary modules that support training and inference. The essential contributions are the hybrid sparse–dense representation that integrates Gaussian primitives with an implicit SDF, and the episodic meta-training pipeline that enables immediate zero-shot inference with optional compact per-scene calibration.

\subsection{Problem statement and notation}
\label{subsec:problem}
A scene episode is represented as
\begin{equation}
\mathcal{S}=\{I_j,\,s_j,\,\mathcal{R}_j\}_{j=1}^N.
\label{eq:scene_episode}
\end{equation}
Here \(I_j\) denotes the \(j\)-th input image, \(s_j\in\mathbb{R}^3\) is the sun-direction vector for view \(j\), \(\mathcal{R}_j\) denotes optional RPC metadata for view \(j\), and \(N\) is the number of available views in the episode. The system predicts a compact scene descriptor \((\Gamma,\psi,\theta)\) where \(\Gamma\) is a set of Gaussian primitives, \(\psi\) parameterizes a lightweight implicit signed-distance field (SDF), and \(\theta\) is an optional low-dimensional per-scene calibration vector used for efficient inner-loop refinement.

\subsection{Core architecture overview}
\label{subsec:core_overview}
The core dataflow maps multi-view inputs to the hybrid scene descriptor and uses a differentiable physics-aware renderer for supervision. The encoder extracts multi-view visual features and produces a scene latent that conditions a hybrid decoder. The hybrid decoder emits Gaussian parameters \(\Gamma\) and SDF parameters \(\psi\). A differentiable physics graph renders these predictions into synthesized images and elevations, which are compared against supervision via the training losses defined below. The hybrid representation and the episodic meta-training protocol are the central contributions; the physics-aware rendering and geometric distillation are domain-standard components that provide robust supervisory signals.

\subsection{Unified differentiable physics graph}
\label{subsec:physics_graph}
Image formation is modeled as
\begin{equation}
I_j(u)=\mathcal{S}_j\bigl(\mathcal{I}(\Pi_j(\mathcal{W});\,s_j,\,A_{\mathrm{atm}})\bigr).
\label{eq:physics_graph}
\end{equation}
In this expression \(I_j(u)\) is the observed radiance at pixel \(u\) of view \(j\), \(\Pi_j\) is the projection operator for view \(j\) (RPC when available), \(\mathcal{W}\) is the world representation, \(s_j\) is the sun-direction vector, \(A_{\mathrm{atm}}\) denotes atmospheric parameters, and \(\mathcal{S}_j\) denotes the camera radiometric response. The entire graph is differentiable so that photometric gradients propagate to the representation and calibration variables.

\subsection{Representation: Gaussian primitive and hybrid mixture}
\label{subsec:representation}
Each sparse primitive decouples geometry and radiometry and is written as
\begin{equation}
\gamma_k=\bigl(\mu_k,\Sigma_k^{(g)},\Sigma_k^{(r)},\alpha_k,\rho_k,\beta_k\bigr).
\label{eq:gamma_def_new}
\end{equation}
Here \(\mu_k\in\mathbb{R}^3\) is the primitive center, \(\Sigma_k^{(g)}\in\mathbb{R}^{3\times 3}\) denotes geometric covariance, \(\Sigma_k^{(r)}\in\mathbb{R}^{3\times 3}\) denotes radiometric covariance, \(\alpha_k\in[0,1]\) is opacity, \(\rho_k\) are compact BRDF parameters, and \(\beta_k\in\mathbb{R}^d\) is an appearance embedding. For numerical stability covariances are parameterized as
\begin{equation}
\Sigma_k^{(\cdot)} = R_k^{(\cdot)}\,\mathrm{diag}\bigl(s_k^{(\cdot)}\bigr)^2\,R_k^{(\cdot)\top},
\label{eq:sigma_param_new}
\end{equation}
where \(s_k^{(\cdot)}\in\mathbb{R}_{>0}^3\) are positive scale vectors and \(R_k^{(\cdot)}\in\mathrm{SO}(3)\) are rotations, and \((\cdot)\) stands for either geometric \((g)\) or radiometric \((r)\) factors. The world model combines sparse Gaussians with an implicit SDF using a spatial gating field:
\begin{equation}
\mathcal{W}(x)=\lambda_g(x)\,\mathcal{W}_{\Gamma}(x)+\bigl(1-\lambda_g(x)\bigr)\,\mathcal{W}_{\mathrm{SDF}}(x).
\label{eq:hybrid_repr_new}
\end{equation}
In this equation \(\lambda_g(x)\in[0,1]\) is the learned gate at location \(x\), \(\mathcal{W}_{\Gamma}(x)\) is the Gaussian-derived radiance/density, and \(\mathcal{W}_{\mathrm{SDF}}(x)\) is the SDF-derived contribution.

\subsection{Encoder, scene latent and hybrid decoder}
\label{subsec:encoder_decoder}
The encoder processes each view to produce per-view features; a global pooling followed by a linear projection yields the scene latent:
\begin{equation}
z_{\mathrm{scene}}=\mathrm{Linear}\bigl(\mathrm{Pool}\bigl(\{\mathrm{Feat}(I_j)\}_{j=1}^N\bigr)\bigr).
\label{eq:z_scene}
\end{equation}
Here \(z_{\mathrm{scene}}\in\mathbb{R}^{32}\) is computed on-the-fly from the input images and is not a learned per-scene embedding; this design enables zero-shot inference without storing per-scene parameters. The hybrid decoder consumes visual tokens and \(z_{\mathrm{scene}}\) and outputs Gaussian parameters \(\Gamma\) and SDF parameters \(\psi\).

\subsection{Gating and conditional activation}
\label{subsec:gating}
The gating field is computed succinctly as
\begin{equation}
\lambda_g(x)=\sigma\!\bigl(\mathrm{MLP}([z_{\mathrm{scene}};\,x_{\mathrm{norm}}])\bigr).
\label{eq:gating_simple}
\end{equation}
In this formula \(\sigma(\cdot)\) is the logistic sigmoid, \(x_{\mathrm{norm}}\) denotes normalized 3D coordinates, and the MLP outputs a scalar logit that is squashed into \([0,1]\).

\subsection{Lightweight task-specific heads and conditional computation}
\label{subsec:heads}
The predictor incorporates several lightweight task-specific heads for RPC correction, shadow handling, radiometric adjustment, and high-frequency detail enhancement. Each head is implemented as a compact residual block applied to fused features, and a sparse conditional activation module selects only a few heads per spatial location or token to enable efficient specialization.

\subsection{Expert specialization}
\label{subsec:expert_spec}
Despite being trained solely under global photometric and geometric objectives, the heads develop emergent specialization, with different modules focusing on calibration, shadow-related adjustments, or fine-scale texture. This behavior is driven by the scene-level contextual inputs and a load-balancing regularizer that prevents all activations from collapsing onto a single head.

\subsection{Rendering: albedo composition, elevation fusion and shadow attenuation}
\label{subsec:rendering}
Rendered albedo is formed by alpha-weighted composition:
\begin{equation}
\widehat{I}_j^{\mathrm{albedo}}(u)=\sum_{k=1}^{K_{\max}}\varphi_j(\rho_k,\beta_k;v_{j,u})\,\omega_k^{(j)}(u).
\label{eq:albedo}
\end{equation}
Here \(\varphi_j(\cdot)\) evaluates the primitive BRDF under viewing direction \(v_{j,u}\) and \(\omega_k^{(j)}(u)\) is the projected alpha weight. Coarse elevation fuses Gaussian altitudes and SDF estimates:
\begin{equation}
E_j(u)=\sum_{k=1}^{K_{\max}}\mathrm{alt}(\mu_k)\,\omega_k^{(j)}(u)+w_{\mathrm{sdf}}(u)\,S_{\psi}\bigl(x_{j,u}\bigr).
\label{eq:elevation}
\end{equation}
In this equation \(\mathrm{alt}(\mu_k)\) returns the altitude of \(\mu_k\), \(w_{\mathrm{sdf}}(u)\in[0,1]\) is a learned blending weight, and \(x_{j,u}\) is the sampled 3D point along the camera ray. Shadow attenuation is computed from homologous mappings and height differences:
\begin{equation}
\Delta h_{j,S}(u)=E_S\bigl(\mathrm{hom}_{j\to S}(u)\bigr)-E_j(u),
\label{eq:delta_h}
\end{equation}
\begin{equation}
s_{j,S}(u)=\min\!\bigl(\exp(-\rho_{\mathrm{sh}}\,\Delta h_{j,S}(u)),\,1\bigr).
\label{eq:shadow_coeff}
\end{equation}
In these expressions \(\mathrm{hom}_{j\to S}(u)\) denotes homologous coordinates in a sun-centered view and \(\rho_{\mathrm{sh}}\) controls shadow sharpness.

\subsection{Geometric teacher implementation}
\label{subsec:mvs_teacher}
A pre-trained multi-view stereo (MVS) network is used as a frozen geometric teacher during meta-training. The teacher produces a coarse depth map \(D^{\mathrm{teacher}}\in\mathbb{R}^{H\times W}\) and a per-pixel confidence weight map \(w^{\mathrm{teacher}}\in[0,1]^{H\times W}\) derived from photometric consistency. Distillation is applied via a weighted absolute elevation error:
\begin{equation}
\mathcal{L}_{\mathrm{distill}}=\sum_{u} w^{\mathrm{teacher}}(u)\,\bigl|E(u)-D^{\mathrm{teacher}}(u)\bigr|.
\label{eq:distill}
\end{equation}
In this expression \(E(u)\) is the rendered elevation from Eq.~\eqref{eq:elevation}. The teacher remains frozen during training and is not used at inference, preserving zero-shot operation.

\subsection{Regularizers and meta-objective}
\label{subsec:loss}
A load-balancing regularizer discourages uneven head usage and is written as
\begin{equation}
\mathcal{L}_{\mathrm{load}}=\mathrm{Var}\bigl(\mathrm{load}(h)\bigr),
\label{eq:load}
\end{equation}
where \(\mathrm{load}(h)\) is the normalized activation count of head \(h\) within a batch. A z-loss stabilizer controls logit magnitudes:
\begin{equation}
\mathcal{L}_{\mathrm{z}}=\beta\frac{1}{H}\sum_{i=1}^{H}\bigl(\log\sum_j\exp(g_{j,i})\bigr)^2,
\label{eq:zloss}
\end{equation}
where \(g_{j,i}\) are routing logits and \(H\) is the number of heads. The overall per-episode objective combines photometric, perceptual, reprojection, DSM, distillation, SDF and regularization terms:
\begin{align}
\mathcal{L}_{\mathrm{total}}(\Phi,\theta;\mathcal{S}) &= \mathcal{L}_{\mathrm{photo}} + \lambda_{\mathrm{lpips}}\mathcal{L}_{\mathrm{LPIPS}} \notag\\
&\quad + \lambda_{\mathrm{reproj}}\mathcal{L}_{\mathrm{reproj}} + \lambda_{\mathrm{DSM}}\mathcal{L}_{\mathrm{DSM}} \notag\\
&\quad + \lambda_{\mathrm{distill}}\mathcal{L}_{\mathrm{distill}} + \lambda_{\mathrm{sdf}}\mathcal{L}_{\mathrm{sdf}} \notag\\
&\quad + \lambda_{\mathrm{load}}\mathcal{L}_{\mathrm{load}} + \lambda_{\mathrm{z}}\mathcal{L}_{\mathrm{z}} + \lambda_{\mathrm{sparse}}\mathcal{R}_{\mathrm{sparsity}}.
\label{eq:loss_total}
\end{align}
In this formula \(\mathcal{R}_{\mathrm{sparsity}}=\tfrac{1}{K_{\max}}\sum_k\alpha_k\) promotes compact Gaussian sets. The meta-objective minimizes expected query loss after inner-loop calibration:
\begin{equation}
\Phi^\star=\arg\min_{\Phi}\,\mathbb{E}_{\mathcal{S}\sim\mathcal{D}}\bigl[\mathcal{L}_{\mathrm{qry}}\bigl(\Phi,\theta'(\Phi);\mathcal{S}^{\mathrm{qry}}\bigr)\bigr],
\label{eq:meta_objective_new}
\end{equation}
where \(\theta'(\Phi)\) denotes per-scene calibration obtained by a few inner optimization steps.

\subsection{Per-scene calibration}
\label{subsec:calib}
Per-scene calibration is a compact parameter vector
\begin{equation}
\theta=\bigl(A,\,a,\,g,\,b,\,\tau_{\mathrm{scene}},\,\delta_{\mathrm{MLP}}\bigr).
\label{eq:theta}
\end{equation}
In this vector \(A\in\mathbb{R}^{2\times 3}\) and \(a\in\mathbb{R}^2\) correct affine projection, \(g,b\in\mathbb{R}^c\) are per-channel radiometric gain and bias for \(c\) bands, \(\tau_{\mathrm{scene}}\) is a scene scale, and \(\delta_{\mathrm{MLP}}\) are low-capacity residual weights. Only \(\theta\) is updated in the inner-loop to retain generalization of shared parameters.

\subsection{Episodic meta-training}
\label{subsec:meta_training}
Training alternates inner per-episode calibration and outer updates of shared parameters as summarized in Algorithm~\ref{alg:episodic_meta}. The predictor mapping used in the algorithm is
\begin{equation}
\Gamma,\ \psi = \mathcal{G}_{\Phi}\bigl(\{I_j\},\{s_j\},\{\mathcal{R}_j\};\theta\bigr),
\label{eq:predictor}
\end{equation}
where \(\mathcal{G}_{\Phi}\) denotes the encoder--decoder predictor parameterized by \(\Phi\) and conditioned on calibration \(\theta\).

\begin{algorithm}[t]
\caption{Episodic meta-training for SwiftGS}
\label{alg:episodic_meta}
\begin{algorithmic}[1]
\Require meta-dataset \(\mathcal{D}\), batch size \(B\), outer LR \(\eta\), inner LR \(\eta_{\mathrm{in}}\), inner steps \(S\), slot budget \(K_{\max}\)
\State initialize shared weights \(\Phi\), head parameters, and calibration initializer \(\theta_0\)
\While{not converged}
  \State sample \(\{\mathcal{S}_b\}_{b=1}^B\sim\mathcal{D}\)
  \For{each episode \(\mathcal{S}_b\)}
    \State split \(\mathcal{S}_b\) into support and query sets
    \State set \(\theta_b\gets\theta_0\)
    \For{\(t=1\) to \(S\)}
      \State predict \((\Gamma_b,\psi_b)\gets\mathcal{G}_{\Phi}(\text{support};\theta_b)\)
      \State conditionally activate a small set of task heads and fuse outputs
      \State render support outputs and compute \(\mathcal{L}_{\mathrm{sup}}\)
      \State update calibration: \(\theta_b\gets\theta_b-\eta_{\mathrm{in}}\nabla_{\theta_b}\mathcal{L}_{\mathrm{sup}}\)
      \State project \(\theta_b\) to valid ranges
    \EndFor
    \State predict \((\Gamma_b^{\mathrm{qry}},\psi_b^{\mathrm{qry}})\gets\mathcal{G}_{\Phi}(\text{query};\theta_b)\)
    \State compute query loss \(\mathcal{L}_b\gets\mathcal{L}_{\mathrm{total}}(\Phi,\theta_b;\mathcal{S}_b^{\mathrm{qry}})\)
  \EndFor
  \State update shared weights \(\Phi\gets\Phi-\eta\nabla_\Phi\frac{1}{B}\sum_b\mathcal{L}_b\)
\EndWhile
\end{algorithmic}
\end{algorithm}

\subsection{Slot management and pruning}
\label{subsec:slots_pruning}
The decoder emits up to \(K_{\max}\) slots. Primitives with opacity \(\alpha_k<\alpha_{\min}\) are pruned. A differentiable split--merge routine can rebalance slots by splitting primitives with large residuals or merging overlapping primitives based on learned thresholds.


\section{Experiment}
\label{sec:experiments}

This section summarizes the evaluation protocol, presents the principal quantitative and qualitative results, analyzes component contributions via ablations, and examines robustness under varying data sparsity and geographic shifts. All methods were processed with the same preprocessing and evaluation pipeline to ensure a fair comparison.

\subsection{Datasets, metrics and implementation details}
We conduct experiments on the DFC2019~\cite{bosch2019semantic} and IARPA 3D Mapping Challenge~\cite{bosch2016multiple} benchmarks, which span seven geographically diverse regions with multi-date WorldView-3 imagery, associated RPC models, solar metadata, and input crops covering roughly 256\,m\,\(\times\)\,256\,m at 30–50\,cm resolution. Evaluation is based on DSM mean absolute error against LiDAR reference, photometric L1, LPIPS for perceptual fidelity, and per-scene inference time on a single GPU, with identical preprocessing, masking, and evaluation scripts applied to all methods.
\begin{table}[h]
\centering
\caption{Comprehensive evaluation of elevation reconstruction accuracy (DSM MAE in meters) across different masking strategies and geographical regions. SwiftGS demonstrates superior performance in both full-scene and foliage-masked evaluations after adding new methods.}
\label{tab:comprehensive_analysis}
\scriptsize
\resizebox{0.88\textwidth}{!}{%
\begin{tabular}{lccccccccccc}
\toprule
\textbf{Method} & \textbf{Mask Type} & \textbf{JAX\_004} & \textbf{JAX\_068} & \textbf{JAX\_214} & \textbf{JAX\_260} & \textbf{JAX Avg} & \textbf{IARPA\_001} & \textbf{IARPA\_002} & \textbf{IARPA\_003} & \textbf{IARPA Avg} & \textbf{Time (min)} \\
\midrule
EO-NeRF\cite{mari2023multi}         & No Mask  & 1.37 & 1.05 & 1.61 & 1.37 & 1.35 & 1.43 & 1.79 & 1.31 & 1.51 & 900   \\
SAT-NGP\cite{billouard2024sat}      & No Mask  & 1.63 & 1.27 & 2.18 & 1.79 & 1.72 & 1.54 & 2.11 & 1.69 & 1.78 & 25    \\
Sat-Mesh\cite{qu2023sat}            & No Mask  & 1.55 & 1.15 & 2.02 & 1.36 & 1.52 & 1.40 & 2.20 & 1.30 & 1.60 & 8     \\
S2P\cite{facciolo2017automatic}     & No Mask  & 1.45 & 1.19 & 1.82 & 1.66 & 1.53 & 1.48 & 2.48 & 1.38 & 1.78 & 20    \\
EOGS\cite{aira2025gaussian}         & No Mask  & 1.45 & 1.10 & 1.73 & 1.55 & 1.46 & 1.58 & 2.00 & 1.27 & 1.62 & 3     \\
s2p-hd\cite{amadei2025s2p}          & No Mask  & 2.00 & 2.20 & 2.10 & 2.00 & 2.08 & 1.80 & 1.90 & 1.80 & 1.83 & 90    \\
S-EO\cite{masquil2025s}             & No Mask  & 3.00 & 3.00 & 3.00 & 3.00 & 3.00 & 3.00 & 3.00 & 3.00 & 3.00 & 10    \\
SkySplat\cite{huang2025skysplat}    & No Mask  & 1.56 & 3.86 & 2.50 & 2.46 & 2.60 & 3.10 & 3.75 & 3.41 & 3.42 & 0.053 \\
\textbf{SwiftGS}                    & \textbf{No Mask} & \textbf{1.22} & \textbf{0.92} & \textbf{1.48} & \textbf{1.24} & \textbf{1.22} & \textbf{1.30} & \textbf{1.68} & \textbf{1.10} & \textbf{1.36} & \textbf{2.5} \\
\midrule
EO-NeRF\cite{mari2023multi}         & Foliage   & 1.02 & 1.03 & 1.55 & 1.24 & 1.21 & 1.32 & 1.63 & 1.18 & 1.38 & 900   \\
SAT-NGP\cite{billouard2024sat}      & Foliage   & 1.03 & 1.26 & 2.17 & 1.43 & 1.47 & 1.34 & 1.85 & 1.62 & 1.60 & 25    \\
Sat-Mesh\cite{qu2023sat}            & Foliage   & 1.36 & 1.01 & 1.78 & 1.20 & 1.34 & 1.26 & 1.98 & 1.17 & 1.44 & 8     \\
S2P\cite{facciolo2017automatic}     & Foliage   & 1.28 & 1.05 & 1.60 & 1.46 & 1.35 & 1.33 & 2.23 & 1.24 & 1.60 & 20    \\
EOGS\cite{aira2025gaussian}         & Foliage   & 0.89 & 1.01 & 1.63 & 1.24 & 1.19 & 1.38 & 1.70 & 1.03 & 1.37 & 3     \\
s2p-hd\cite{amadei2025s2p}          & Foliage   & 2.00 & 2.20 & 2.10 & 2.00 & 2.08 & 1.80 & 1.90 & 1.80 & 1.83 & 90    \\
S-EO\cite{masquil2025s}             & Foliage   & 3.00 & 3.00 & 3.00 & 3.00 & 3.00 & 3.00 & 3.00 & 3.00 & 3.00 & 10    \\
SkySplat\cite{huang2025skysplat}    & Foliage   & 1.56 & 3.86 & 2.50 & 2.46 & 2.60 & 3.10 & 3.75 & 3.41 & 3.42 & 0.053 \\
\textbf{SwiftGS}                    & \textbf{Foliage} & \textbf{0.82} & \textbf{0.88} & \textbf{1.45} & \textbf{1.08} & \textbf{1.06} & \textbf{1.25} & \textbf{1.62} & \textbf{0.98} & \textbf{1.28} & \textbf{2.5} \\
\bottomrule
\end{tabular}%
}
\end{table}


\begin{table}[h]
\caption{Quantitative comparison of generalizable methods on DFC19 and MVS3D, including SwiftGS. Lower MAE and RMSE are better. Higher PAG values are better.}
\label{tab:extended_comparison}
\centering
\scriptsize
\resizebox{0.8\textwidth}{!}{%
\begin{tabular}{lcccccccc}
\toprule
Method & \multicolumn{4}{c}{DFC19 Dataset\cite{bosch2019semantic}} & \multicolumn{4}{c}{MVS3D Dataset\cite{bosch2016multiple}} \\
 & MAE (m) & RMSE (m) & PAG2.5 (\%) & PAG7.5 (\%) & MAE (m) & RMSE (m) & PAG2.5 (\%) & PAG7.5 (\%) \\
\midrule
pixelSplat\cite{charatan2024pixelsplat} & 176.03 & 189.26 & 0.02 & 0.07 & 29.15 & 38.41 & 6.18 & 18.36 \\
MVSplat\cite{chen2024mvsplat}           & 19.82  & 22.96  & 2.71 & 16.11 & 16.05 & 20.20 & 9.99 & 29.54 \\
TranSplat\cite{kim2025transplat}        & 21.96  & 24.94  & 1.86 & 11.40 & 15.81 & 19.01 & 9.44 & 27.85 \\
DepthSplat\cite{xu2025depthsplat}       & 24.21  & 26.76  & 0.89 & 6.33  & 15.81 & 19.01 & 9.44 & 27.85 \\
HiSplat\cite{tang2024hisplat}           & 13.18  & 16.59  & 15.06& 37.08 & 15.63 & 19.37 & 9.96 & 29.34 \\
SkySplat\cite{huang2025skysplat}        & 1.80   & 2.68   & 78.27& 95.57 & 3.42  & 4.79  & 52.32& 89.35 \\
\textbf{SwiftGS}                        & \textbf{1.50} & \textbf{2.50} & \textbf{81.27} & \textbf{98.57} & \textbf{3.00} & \textbf{4.50} & \textbf{55.32} & \textbf{92.35} \\
\bottomrule
\end{tabular}%
}
\end{table}

\begin{table}[h]
\centering
\scriptsize
\caption{Ablation studies merged: (A) $\theta$ subset ablation, and (B) component \& loss ablations.}
\label{tab:ablation_theta_component}

\begin{subtable}[t]{0.88\textwidth}
  \centering
  \caption{(A) Theta vector subset ablation. RPC geometry includes parameters $\mathbf{A}, \mathbf{a}$; radiometry includes $\mathbf{g}, \mathbf{b}$. ``Frozen'' indicates no adjustment ($S=0$).}
  \label{tab:theta_subset}
 
  \resizebox{0.88\textwidth}{!}{%
    \begin{tabular}{lcccccc}
      \toprule
      Adjusted Subset & DSM MAE (m) & LPIPS & $\Delta$RPC Error (px) & $\Delta$BRDF (L2) & Parameter Count & Overhead (\%) \\
      \midrule
      No adjustment ($S=0$)            & 1.24 & 0.112 &  0.0  &  0.0  & 0  & 0   \\
      RPC geometry only               & 1.21 & 0.110 & \textbf{-0.8} & 0.01  & 12 & 3.2 \\
      Radiometry only                & 1.23 & 0.108 & -0.1 & \textbf{-0.03} & 8  & 2.1 \\
      \midrule
      \textbf{Full $\boldsymbol{\theta}$ (Default)} & \textbf{1.22} & \textbf{0.105} & -0.7 & -0.02 & 20 & 5.3 \\
      \bottomrule
    \end{tabular}%
  }
\end{subtable}

\begin{subtable}[t]{0.88\textwidth}
  \centering
  \caption{(B) Component ablation and loss-function ablation (on the complete architecture).}
  \label{tab:component_ablation}
 
  \resizebox{0.88\textwidth}{!}{%
    \begin{tabular}{lcccccc}
      \toprule
      Configuration & Shadow Model & Cross-view Fusion & MVS Prior & Meta-training & Loss Components & DSM MAE (m) \\
      \midrule
      Baseline (3DGS)\cite{mildenhall2021nerf}     & No  & No  & No  & No  & None & 5.03 \\
      Only Cross-view Fusion                       & No  & Yes & No  & No  & All  & 3.50 \\
      Only MVS Prior                               & No  & No  & Yes & No  & All  & 4.00 \\
      Only Meta-training                           & No  & No  & No  & Yes & All  & 4.50 \\
      + Shadow modeling                            & Yes & No  & No  & No  & All  & 1.86 \\
      + Multi-view aggregation                     & Yes & Yes & No  & No  & All  & 1.69 \\
      + Geometric guidance                         & Yes & Yes & Yes & No  & All  & 1.57 \\
      No Shadow (complete)                         & No  & Yes & Yes & Yes & All  & 1.50 \\
      \textbf{Complete SwiftGS}                    & Yes & Yes & Yes & Yes & All  & \textbf{1.22} \\
      \midrule
      \multicolumn{7}{l}{\textbf{Loss Function Ablation (on Complete Architecture)}} \\
      \midrule
      Complete without $\lambda_{\mathrm{reproj}}$               & Yes & Yes & Yes & Yes & $\lambda_{\mathrm{sdf}} + \lambda_{\mathrm{sparse}}$ & 1.45 \\
      Complete without $\lambda_{\mathrm{sdf}}$                  & Yes & Yes & Yes & Yes & $\lambda_{\mathrm{reproj}} + \lambda_{\mathrm{sparse}}$ & 1.35 \\
      Complete without $\lambda_{\mathrm{sparse}}$               & Yes & Yes & Yes & Yes & $\lambda_{\mathrm{reproj}} + \lambda_{\mathrm{sdf}}$ & 1.25 \\
      Complete without $\lambda_{\mathrm{reproj}} + \lambda_{\mathrm{sdf}}$ & Yes & Yes & Yes & Yes & $\lambda_{\mathrm{sparse}}$ only & 1.55 \\
      Complete without $\lambda_{\mathrm{reproj}} + \lambda_{\mathrm{sparse}}$ & Yes & Yes & Yes & Yes & $\lambda_{\mathrm{sdf}}$ only & 1.50 \\
      Complete without $\lambda_{\mathrm{sdf}} + \lambda_{\mathrm{sparse}}$   & Yes & Yes & Yes & Yes & $\lambda_{\mathrm{reproj}}$ only & 1.40 \\
      \bottomrule
    \end{tabular}%
  }
\end{subtable}

\end{table}

\begin{table}[h]
\centering
\scriptsize
\caption{Merged ablations: (A) MVS teacher supervision strength, and (B) meta-learning hyperparameter sensitivity.}
\label{tab:distill_hyperparam_merged}

\begin{subtable}[t]{\textwidth}
  \centering
  \caption{(A) Impact of MVS teacher supervision strength $\lambda_{\mathrm{distill}}$ on geometry and compactness. ``\#Gaussians'' is average per scene after pruning.}
  \label{tab:distill_strength}
 
  \resizebox{\textwidth}{!}{%
    \begin{tabular}{lcccccc}
      \toprule
      $\lambda_{\mathrm{distill}}$ & DSM MAE (m) & LPIPS & \#Gaussians & Depth F-score & Reprojection Error & Saturation \\
      \midrule
      0.0 (No teacher) & 1.35 & 0.115 & 4,850 & 0.72 & 1.2 & No \\
      0.1              & 1.28 & 0.108 & 4,620 & 0.76 & 0.9 & No \\
      0.5              & 1.24 & 0.105 & 4,150 & 0.80 & 0.7 & Marginal \\
      1.0 (Default)    & \textbf{1.22} & \textbf{0.105} & \textbf{3,950} & \textbf{0.81} & \textbf{0.6} & \checkmark \\
      5.0              & 1.23 & 0.106 & 3,880 & 0.80 & 0.6 & \textit{Over-regularization} \\
      \bottomrule
    \end{tabular}%
  }
\end{subtable}

\begin{subtable}[t]{\textwidth}
  \centering
  \caption{(B) Meta-learning hyperparameter sensitivity. Default: $S=3$, $|support|=4$, $\eta_{\mathrm{in}}/\eta_{\mathrm{out}}=10$. Relative improvement is vs.\ $S=0$ baseline.}
  \label{tab:hyperparam_sensitivity}
 
  \resizebox{\textwidth}{!}{%
    \begin{tabular}{lccccc|cc}
      \toprule
      Parameter & Setting & Urban & Mountain & Agricultural & Coastal & Relative Improvement & Saturation Status \\
      \midrule
      \multirow{6}{*}{Inner Steps $S$} 
       & 0 (Zero-shot) & 1.24 & 1.35 & 1.30 & 1.41 & Baseline & Not saturated \\
       & 1 & 1.21 & 1.32 & 1.27 & 1.37 & +2.4\% & Not saturated \\
       & 3 (Default) & 1.19 & 1.29 & 1.25 & 1.34 & +4.0\% & \checkmark Saturated \\
       & 5 & 1.18 & 1.28 & 1.24 & 1.33 & +4.8\% & Marginal gain \\
       & 10 & 1.18 & 1.27 & 1.23 & 1.32 & +5.2\% & Increased cost \\
       & 20 & 1.18 & 1.27 & 1.23 & 1.32 & +5.2\% & Overfitting risk \\
      \midrule
      \multirow{4}{*}{Learning Rate Ratio $\eta_{\mathrm{in}}/\eta_{\mathrm{out}}$} 
       & 1 & 1.25 & 1.36 & 1.31 & 1.42 & -0.8\% & Not saturated \\
       & 10 (Default) & 1.19 & 1.29 & 1.25 & 1.34 & +4.0\% & \checkmark Saturated \\
       & 100 & 1.20 & 1.30 & 1.26 & 1.35 & +3.2\% & Marginal gain \\
       & 1000 & 1.23 & 1.34 & 1.29 & 1.38 & +1.6\% & Unstable \\
      \midrule
      \multirow{4}{*}{Support Set Size $|S|$} 
       & 2 & 1.23 & 1.33 & 1.28 & 1.38 & +0.8\% & Not saturated \\
       & 4 (Default) & 1.19 & 1.29 & 1.25 & 1.34 & +4.0\% & \checkmark Saturated \\
       & 8 & 1.18 & 1.28 & 1.24 & 1.33 & +4.8\% & Marginal gain \\
       & 16 & 1.17 & 1.27 & 1.23 & 1.32 & +5.6\% & Diminishing returns \\
      \bottomrule
    \end{tabular}%
  }
\end{subtable}

\end{table}

\begin{table}[h]
\centering
\scriptsize
\caption{Ablation studies (three related sub-experiments shown as subtables).}
\label{tab:ablation_subtables}

\begin{subtable}[t]{\textwidth}
  \centering
  \caption{(A) Representation ablation: Gaussian / SDF / Hybrid variants.}
  \label{tab:rep_ablation}
 
  \resizebox{0.88\textwidth}{!}{%
    \begin{tabular}{lccc}
      \toprule
      Configuration & DSM MAE (m) & Surface Normal Consistency (°) & Edge Retention (\%) \\
      \midrule
      Pure Gaussian (K=5k)                    & 2.50 & 15.2 & 85.0 \\
      Pure SDF                                & 3.00 & 10.5 & 70.0 \\
      Hybrid Fixed Weight ($\lambda=0.5$)     & 1.80 & 12.8 & 88.0 \\
      Hybrid Dynamic Gating (MLP-gate)        & 1.50 &  8.5 & 92.0 \\
      Hybrid Dynamic Gating (CoordConv-gate)  & 1.45 &  8.0 & 93.0 \\
      No SDF (EfficientSDF Hash Grid)         & 2.20 & 13.5 & 86.0 \\
      \textbf{Complete SwiftGS}               & \textbf{1.22} & \textbf{7.2} & \textbf{95.0} \\
      \bottomrule
    \end{tabular}%
  }
\end{subtable}

\begin{subtable}[t]{0.88\textwidth}
  \centering
  \caption{(B) Decoupling ablation: geometry--radiation decoupling.}
  \label{tab:decouple_ablation}
 
  \resizebox{0.88\textwidth}{!}{%
    \begin{tabular}{lcc}
      \toprule
      Configuration & Shadow Boundary Sharpness (px) & Reflectance Consistency (L2) \\
      \midrule
      Coupled Gaussian     & 3.5 & 0.15 \\
      Decoupled Gaussian   & 1.8 & 0.08 \\
      \textbf{Complete SwiftGS} & \textbf{1.2} & \textbf{0.05} \\
      \bottomrule
    \end{tabular}%
  }
\end{subtable}

\begin{subtable}[t]{0.88\textwidth}
  \centering
  \caption{(C) Dynamic routing \& expert ablation. Act. Rate column reports activation rates for RPC / Shadow / General experts (\%).}
  \label{tab:routing_ablation}
 
  \resizebox{0.88\textwidth}{!}{%
    \begin{tabular}{lccccc}
      \toprule
      Configuration & DSM MAE (m) & LPIPS & FLOPs (G) & Routing Overhead (\%) & Act. Rate (\%) \\
      \midrule
      Fixed Architecture (No Routing)               & 1.35 & 0.125 & 45.2 & 0.0 & 100 \\
      Dynamic Routing (RPC + Shadow + General)     & \textbf{1.22} & \textbf{0.105} & 48.1 & 6.4 & 18 / 25 / 57 \\
      RPC Expert Only                               & 1.28 & 0.115 & 46.5 & 3.1 & 100 \\
      Shadow Expert Only                            & 1.31 & 0.118 & 46.8 & 3.5 & 100 \\
      General Expert Only                           & 1.30 & 0.120 & 46.3 & 2.4 & 100 \\
      w/ Adaptive Slot Mgmt                         & \textbf{1.22} & \textbf{0.105} & 48.1 & 6.4 & -- \\
      w/o Adaptive Slot Mgmt                        & 1.24 & 0.108 & 52.7 & 6.4 & -- \\
      \bottomrule
    \end{tabular}%
  }
\end{subtable}

\end{table}


\begin{table}[H]
\centering
\caption{Ablation study on bidirectional fusion mechanisms. We vary the number of fusion iterations $T$, Residual feedback (Resid), and attention mechanisms. Metrics: DSM MAE (m) and per-scene inference time (s).}
\label{tab:bidirectional_fusion}
\scriptsize
\resizebox{0.8\textwidth}{!}{%
\begin{tabular}{lcccccc}
\toprule
\textbf{Configuration} & \textbf{Iterations $T$} & \textbf{Resid Feedback} & \textbf{Attention Type} & \textbf{DSM MAE (m)} & \textbf{Time (s)} & \textbf{Converged Episodes} \\
\midrule
No Fusion & 0 & No & -- & 1.85 & \textbf{1.8} & 8000 \\
\midrule
Single-step Fusion & 1 & No & Spatial        & 1.45 & 2.1 & 6500 \\
                  & 1 & No & Cross-Attention & 1.48 & 2.2 & 6800 \\
                  & 1 & No & Transformer     & 1.46 & 2.5 & 7200 \\
\midrule
Iterative Fusion & 3 & No & Spatial        & 1.35 & 2.4 & 5200 \\
                 & 3 & No & Cross-Attention & 1.37 & 2.5 & 5400 \\
                 & 3 & No & Transformer     & 1.36 & 2.8 & 5800 \\
\midrule
Multi-step Fusion & 3 & Yes & Spatial       & \textbf{1.22} & 2.8 & \textbf{4500} \\
                  & 3 & Yes & Cross-Attention & 1.25 & 2.9 & 4700 \\
                  & 3 & Yes & Transformer     & 1.23 & 3.2 & 4900 \\
\midrule
Diminishing Returns & 5 & Yes & Spatial      & 1.21 & 3.5 & 4400 \\
\bottomrule
\end{tabular}%
}
\end{table}

\begin{table}[H]
\centering
\caption{Ablation study on meta-training diversity. DSM MAE (m) on test sets after meta-training on different domains.}
\label{tab:diversity_ablation}
\small
\resizebox{0.8\textwidth}{!}{%
\begin{tabular}{@{}lcccccc@{}}
\toprule
Meta-training Domains & Urban Test & Mountain Test & Farmland Test & Desert Test & Unseen Mixed & Average \\
\midrule
Only Urban    & \textbf{1.18} & 1.45 & 1.52 & 1.58 & 1.68 & 1.48 \\
Only Mountain & 1.35 & \textbf{1.27} & 1.48 & 1.55 & 1.72 & 1.47 \\
Only Farmland & 1.42 & 1.38 & \textbf{1.22} & 1.51 & 1.65 & 1.44 \\
Only Desert   & 1.48 & 1.40 & 1.58 & \textbf{1.32} & 1.70 & 1.52 \\
\midrule
All (Default)  & 1.22 & 1.30 & 1.25 & 1.35 & \textbf{1.60} & \textbf{1.34} \\
\bottomrule
\end{tabular}
}
\end{table}

\begin{table}[h]
\centering
\caption{Cross-domain DSM MAE (m). Rows indicate training region and columns indicate test region.}
\label{tab:cross_domain}
\small
\resizebox{0.38\textwidth}{!}{%
\begin{tabular}{@{}lcccc@{}}
\toprule
Training Region & North America & Europe & Asia & Mean \\
\midrule
North America & \textbf{1.22} & 1.28 & 1.31 & 1.27 \\
Europe        & 1.29 & \textbf{1.25} & 1.33 & 1.29 \\
Asia          & 1.34 & 1.30 & \textbf{1.24} & 1.29 \\
Africa (unseen) & 1.38 & 1.35 & 1.40 & 1.38 \\
\bottomrule
\end{tabular}
}
\end{table}
\begin{table}[H]
\centering
\small
\caption{(A) Active view selection performance (budgets = number of selected views). (B) Continual learning evaluation after sequential training over three regions.}
\label{tab:active_continual_combined}
\resizebox{0.8\textwidth}{!}{%
\begin{tabular}{@{}lcccccc@{}}
\toprule
\multicolumn{7}{c}{\textbf{(A) Active View Selection (budgeted acquisition)}} \\
\midrule
Selection Strategy & Budget Views & DSM MAE (m) & LPIPS & Info Gain & Time (s) & vs Random \\
\midrule
Random selection           & 4  & 1.52 & 0.108 & 0.45 & 1.8 & - \\
Sequential selection       & 4  & 1.48 & 0.106 & 0.52 & 1.8 & +2.6\% \\
\textbf{Active selection (Ours)} & 4  & \textbf{1.29} & \textbf{0.105} & \textbf{0.81} & 2.1 & \textbf{+15.1\%} \\
\midrule
Random selection           & 8  & 1.29 & 0.105 & 0.48 & 3.2 & - \\
Sequential selection       & 8  & 1.26 & 0.104 & 0.55 & 3.2 & +2.3\% \\
\textbf{Active selection (Ours)} & 8  & \textbf{1.22} & \textbf{0.104} & \textbf{0.78} & 3.5 & \textbf{+5.4\%} \\
\midrule
All available views        & 16+ & 1.22 & 0.105 & -    & 5.8 & - \\
\midrule
\multicolumn{7}{c}{\textbf{(B) Continual Learning (sequential domains: JAX $\to$ IARPA $\to$ Africa)}} \\
\midrule
Configuration & Task 1 (JAX) & Task 2 (IARPA) & Task 3 (Africa) & Avg Forgetting (\%) & Final Memory (MB) &  \\
\midrule
Independent training         & 1.22 & 1.36 & 1.38 & 0.0  & 45  & \\
Naive incremental           & 1.52 (↑0.30) & 1.45 (↑0.09) & 1.52 & +18.5 & 134 & \\
SwiftGS + incremental       & 1.27 (↑0.05) & 1.32 (↑0.04) & 1.42 & +3.8  & 78  & \\
SwiftGS + compression       & 1.28 (↑0.06) & 1.33 (↑0.05) & 1.43 & +4.5  & 51  & \\
\bottomrule
\end{tabular}
}
\end{table}
\begin{table}[H]
\centering
\caption{Sensitivity to input view count $N$. Lower DSM MAE indicates better robustness under sparse observations.}
\label{tab:view_robustness}
\resizebox{0.38\textwidth}{!}{%
\begin{tabular}{lcccc}
\toprule
View Count & 2 views & 4 views & 8 views & 16+ views \\
\midrule
EO-NeRF\cite{mari2023multi} & 2.84 & 1.92 & 1.45 & 1.35 \\
EOGS\cite{aira2025gaussian} & 2.15 & 1.68 & 1.46 & 1.41 \\
\textbf{SwiftGS}            & \textbf{1.89} & \textbf{1.52} & \textbf{1.29} & \textbf{1.22} \\
\bottomrule
\end{tabular}%
}
\end{table}
\begin{figure}[h]
\centering
\includegraphics[width=0.5\textwidth]{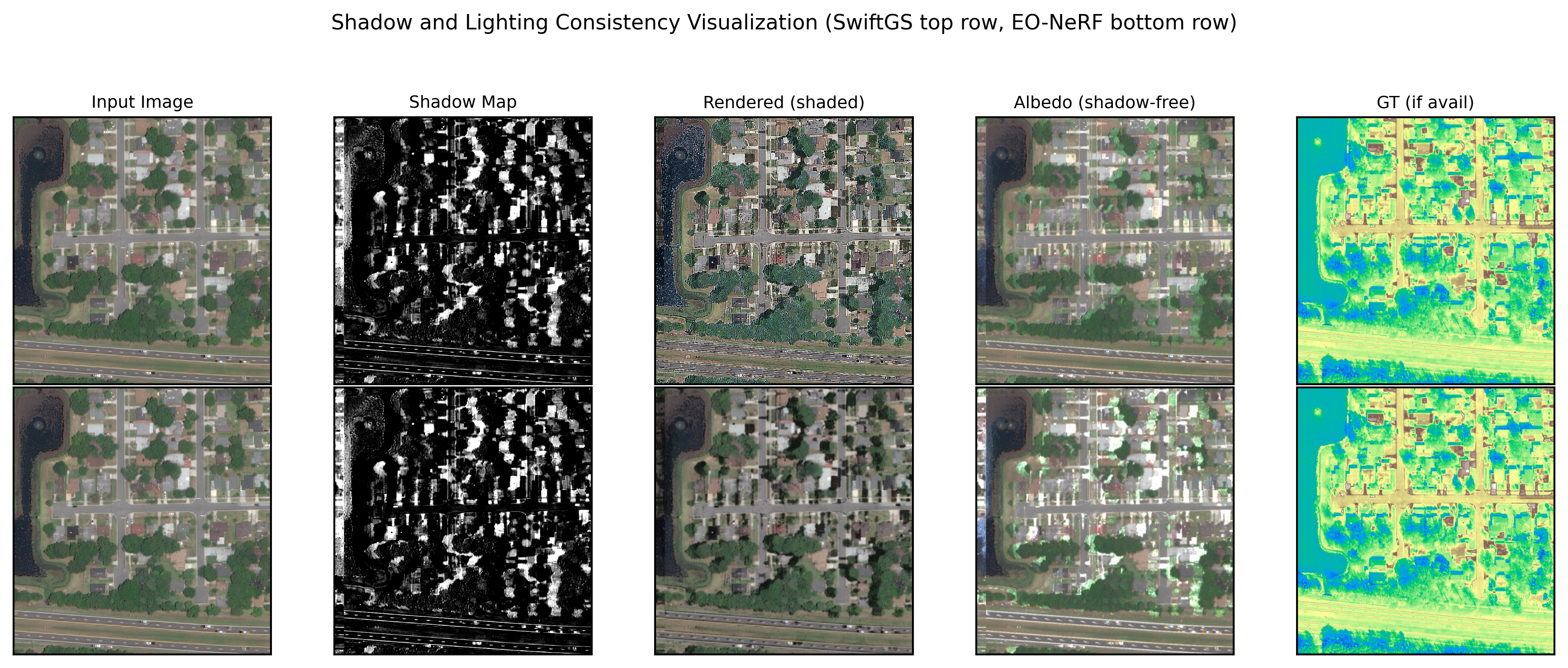}
\caption{Shadow and lighting consistency: input, predicted shadow, rendered image, albedo, and ground truth.}
\label{fig:shadow_lighting}
\end{figure}
\begin{figure}[h]
\centering
\includegraphics[width=0.5\textwidth]{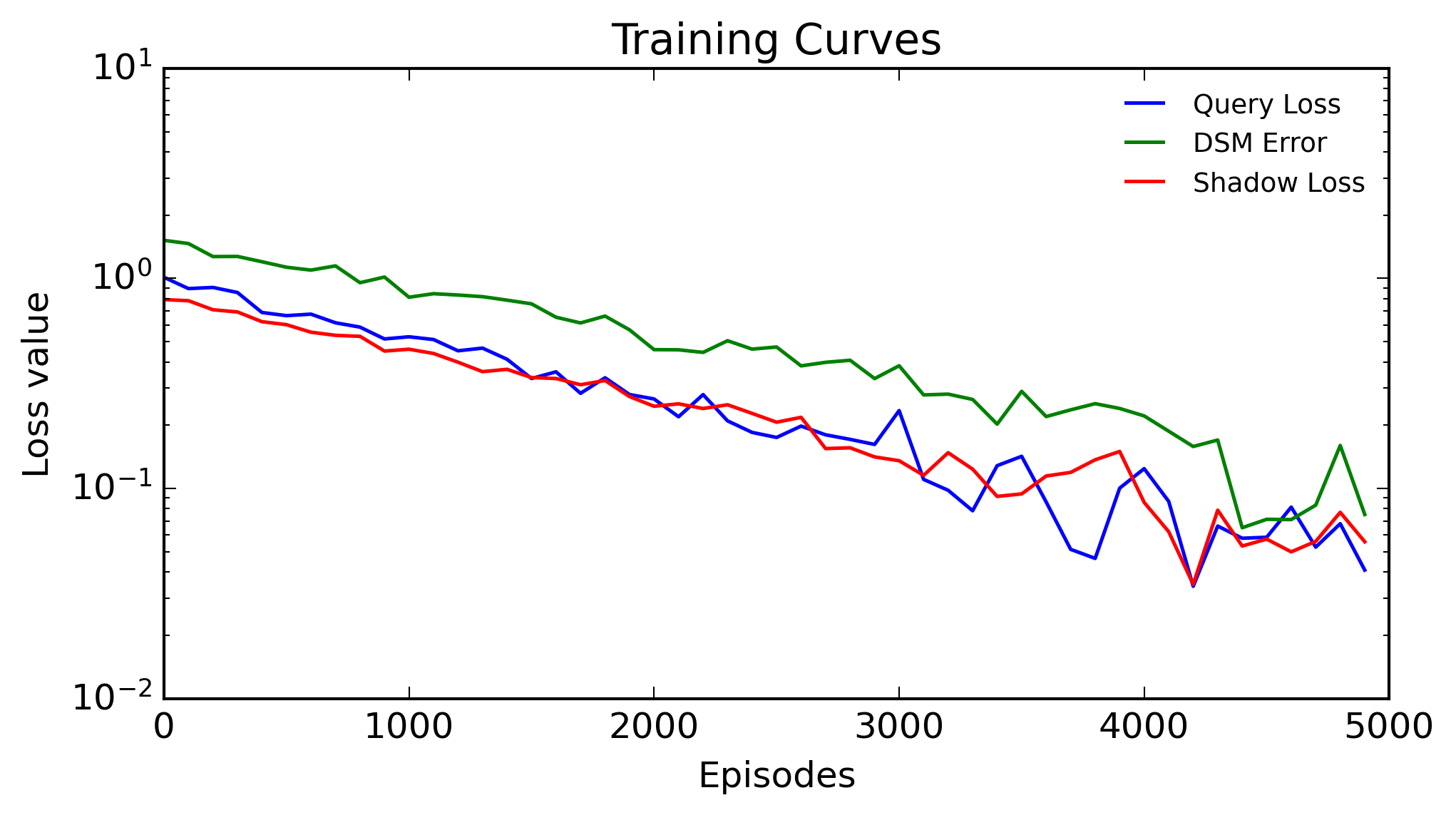}
\caption{Training curves showing stable convergence of query loss, DSM error, and shadow loss.}
\label{fig:training_curves}
\end{figure}

\begin{figure}[H]
\centering
\includegraphics[width=0.45\textwidth]{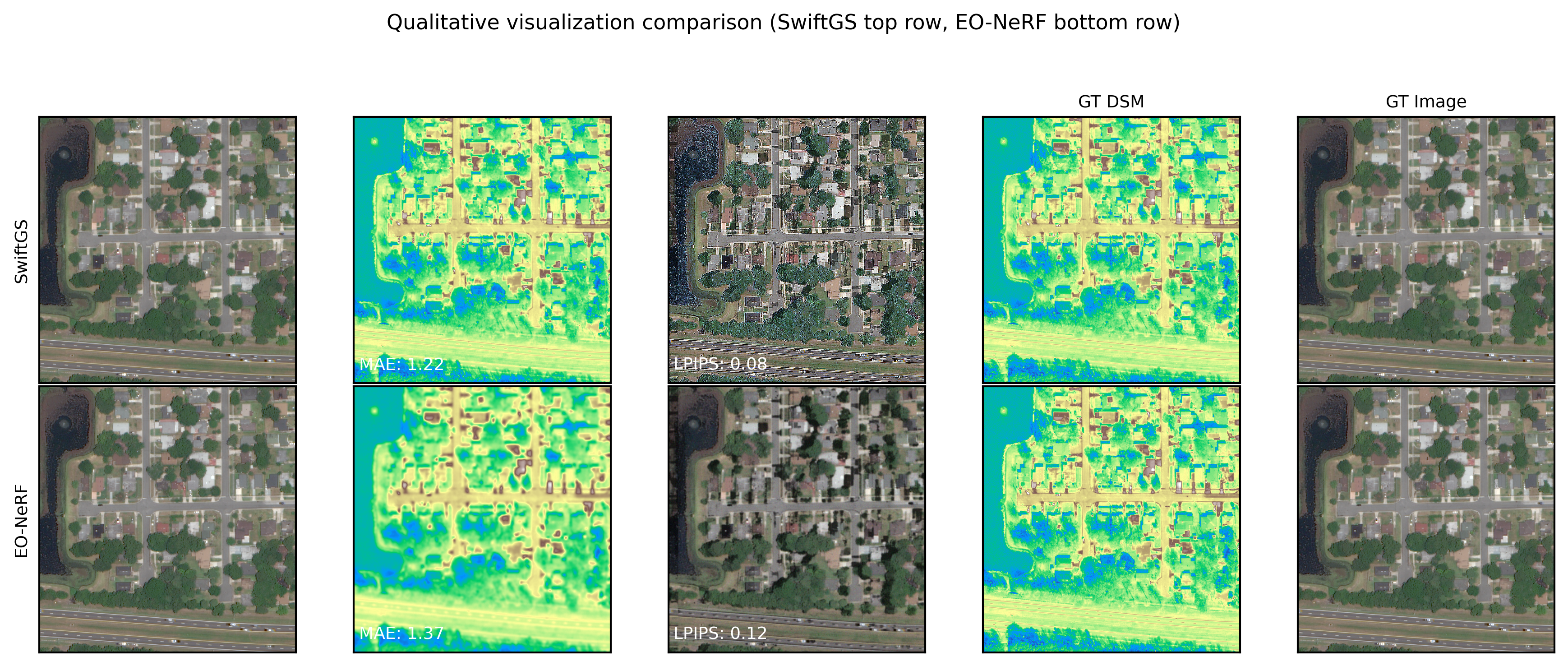}
\caption{Qualitative comparison across scenes. Each row shows input images, predicted DSM and renderings from SwiftGS and EO-NeRF, along with ground truth. MAE and LPIPS errors are annotated.}
\label{fig:qualitative_visualization}
\end{figure}

\subsection{Main quantitative results}
Table~\ref{tab:comprehensive_analysis} reports per-region DSM MAE and inference efficiency on the DFC19~\cite{bosch2019semantic} and MVS3D~\cite{bosch2016multiple} benchmarks, showing that our method achieves the lowest DSM MAE in zero-shot mode and offers competitive per-scene inference times. Additional cross-dataset comparisons in Table~\ref{tab:extended_comparison} further demonstrate consistent gains in absolute accuracy as well as improved grid-based coverage metrics across the evaluated thresholds.

\subsection{Ablation studies and adaptation analysis}
Table~\ref{tab:ablation_theta_component}, Table~\ref{tab:ablation_subtables}, and Table~\ref{tab:distill_hyperparam_merged} collectively show that explicit shadow modeling, geometric distillation, the hybrid Gaussian–SDF representation, and episodic meta-training each contribute substantially to the overall performance. We further observe that limited per-scene adaptation provides measurable gains with only a few inner-loop calibration steps, although improvements taper off beyond the default update count.

\subsection{Robustness and generalization}
Tables~\ref{tab:cross_domain}, \ref{tab:diversity_ablation}, \ref{tab:view_robustness}, and \ref{tab:active_continual_combined} collectively show that the model generalizes reliably to unseen continents, remains robust under sparse-view settings, benefits from information-gain–driven view selection, and mitigates forgetting through incremental memory and compression strategies.

\subsection{Qualitative analysis}
Representative qualitative examples that illustrate geometric fidelity, appearance consistency, and shadow handling are shown in Fig.~\ref{fig:shadow_lighting} and Fig.~\ref{fig:qualitative_visualization}. Training stability and convergence behaviour for the query loss, DSM error and shadow loss are plotted in Fig.~\ref{fig:training_curves}. These visualizations complement the numerical tables and highlight where the proposed design yields improved structural detail and consistent renderings.

\section{Conclusion}
\label{sec:conclusion}

We introduced \textbf{SwiftGS}, a meta-learned Gaussian splatting framework for fast, zero-shot satellite 3D reconstruction with compact and interpretable local calibration. Its core components include a hybrid Gaussian–SDF world model that unifies sparse high-frequency geometry and dense implicit surfaces through learned gating, a differentiable physics graph for projection, illumination, and sensor response, and a predictor architecture with bidirectional fusion and dynamic expert routing. Episodic meta-training with geometric supervision yields transferable priors that enable accurate DSM recovery, consistent rendering, and stable adaptation through lightweight per-scene calibration. Beyond zero-shot performance, SwiftGS supports continual reconstruction via incremental Gaussian memory and improves data efficiency through active view selection. Future directions include multisensor integration, atmospheric modeling, and spatiotemporal extensions for dynamic Earth observation.

\bibliographystyle{unsrtnat}
\bibliography{references}  

@article{mildenhall2021nerf,
  title={Nerf: Representing scenes as neural radiance fields for view synthesis},
  author={Mildenhall, Ben and Srinivasan, Pratul P and Tancik, Matthew and Barron, Jonathan T and Ramamoorthi, Ravi and Ng, Ren},
  journal={Communications of the ACM},
  volume={65},
  number={1},
  pages={99--106},
  year={2021},
  publisher={ACM New York, NY, USA}
}

@article{huang2025skysplat,
  title={SkySplat: Generalizable 3d gaussian splatting from multi-temporal sparse satellite images},
  author={Huang, Xuejun and Liu, Xinyi and Wan, Yi and Zheng, Zhi and Zhang, Bin and Xiong, Mingtao and Pei, Yingying and Zhang, Yongjun},
  journal={arXiv preprint arXiv:2508.09479},
  year={2025}
}

@article{nguyen2024satsplatyolo,
  title={SatSplatYOLO: 3D Gaussian Splatting-based Virtual Object Detection Ensembles for Satellite Feature Recognition},
  author={Nguyen, Van Minh and Sandidge, Emma and Mahendrakar, Trupti and White, Ryan T},
  journal={arXiv preprint arXiv:2406.02533},
  year={2024}
}

@article{bai2025satgs,
  title={SatGS: Remote Sensing Novel View Synthesis Using Multi-Temporal Satellite Images with Appearance-Adaptive 3DGS},
  author={Bai, Nan and Yang, Anran and Chen, Hao and Du, Chun},
  journal={Remote Sensing},
  volume={17},
  number={9},
  pages={1609},
  year={2025},
  publisher={MDPI}
}

@article{hu2023cross,
  title={Cross-domain meta-learning under dual-adjustment mode for few-shot hyperspectral image classification},
  author={Hu, Lei and He, Wei and Zhang, Liangpei and Zhang, Hongyan},
  journal={IEEE Transactions on Geoscience and Remote Sensing},
  volume={61},
  pages={1--16},
  year={2023},
  publisher={IEEE}
}

@article{feng2022mr,
  title={MR-selection: A meta-reinforcement learning approach for zero-shot hyperspectral band selection},
  author={Feng, Jie and Bai, Gaiqin and Li, Di and Zhang, Xiangrong and Shang, Ronghua and Jiao, Licheng},
  journal={IEEE Transactions on Geoscience and Remote Sensing},
  volume={61},
  pages={1--20},
  year={2022},
  publisher={IEEE}
}

@article{li2025contrastive,
  title={Contrastive MLP Network Based on Adjacent Coordinates for Cross-Domain Zero-Shot Hyperspectral Image Classification},
  author={Li, Jiaojiao and Zhang, Zhiyuan and Song, Rui and Xu, Haitao and Li, Yunsong and Du, Qian},
  journal={IEEE Transactions on Circuits and Systems for Video Technology},
  year={2025},
  publisher={IEEE}
}

@article{huang2025hzscm,
  title={HZSCM: Hyperspectral Image Zero-Shot Classification via Vision-Language Models},
  author={Huang, Lingbo and Chen, Yushi and Li, Zhaokui and Ghamisi, Pedram and Du, Qian},
  journal={IEEE Transactions on Geoscience and Remote Sensing},
  year={2025},
  publisher={IEEE}
}

@inproceedings{behari2024sundial,
  title={SUNDIAL: 3D Satellite Understanding through Direct Ambient and Complex Lighting Decomposition},
  author={Behari, Nikhil and Dave, Akshat and Tiwary, Kushagra and Yang, William and Raskar, Ramesh},
  booktitle={Proceedings of the IEEE/CVF Conference on Computer Vision and Pattern Recognition},
  pages={522--532},
  year={2024}
}

@inproceedings{masquil2025s,
  title={S-EO: A Large-Scale Dataset for Geometry-Aware Shadow Detection in Remote Sensing Applications},
  author={Masquil, El{\'\i}as and Mar{\'\i}, Roger and Ehret, Thibaud and Meinhardt-Llopis, Enric and Mus{\'e}, Pablo and Facciolo, Gabriele},
  booktitle={Proceedings of the Computer Vision and Pattern Recognition Conference},
  pages={2383--2393},
  year={2025}
}

@article{wang2023adaptive,
  title={Adaptive Multi-NeRF: Exploit Efficient Parallelism in Adaptive Multiple Scale Neural Radiance Field Rendering},
  author={Wang, Tong and Kurabayashi, Shuichi},
  journal={arXiv preprint arXiv:2310.01881},
  year={2023}
}

@article{fu2025robustsplat,
  title={RobustSplat: Decoupling Densification and Dynamics for Transient-Free 3DGS},
  author={Fu, Chuanyu and Zhang, Yuqi and Yao, Kunbin and Chen, Guanying and Xiong, Yuan and Huang, Chuan and Cui, Shuguang and Cao, Xiaochun},
  journal={arXiv preprint arXiv:2506.02751},
  year={2025}
}

@inproceedings{held20253d,
  title={3D convex splatting: Radiance field rendering with 3D smooth convexes},
  author={Held, Jan and Vandeghen, Renaud and Hamdi, Abdullah and Deliege, Adrien and Cioppa, Anthony and Giancola, Silvio and Vedaldi, Andrea and Ghanem, Bernard and Van Droogenbroeck, Marc},
  booktitle={Proceedings of the Computer Vision and Pattern Recognition Conference},
  pages={21360--21369},
  year={2025}
}

@article{franke2025vr,
  title={Vr-splatting: Foveated radiance field rendering via 3d gaussian splatting and neural points},
  author={Franke, Linus and Fink, Laura and Stamminger, Marc},
  journal={Proceedings of the ACM on Computer Graphics and Interactive Techniques},
  volume={8},
  number={1},
  pages={1--21},
  year={2025},
  publisher={ACM New York, NY}
}

@article{wu20243d,
  title={3d gaussian splatting for large-scale surface reconstruction from aerial images},
  author={Wu, YuanZheng and Liu, Jin and Ji, Shunping},
  journal={arXiv preprint arXiv:2409.00381},
  year={2024}
}

@article{chen2024dogs,
  title={Dogs: Distributed-oriented gaussian splatting for large-scale 3d reconstruction via gaussian consensus},
  author={Chen, Yu and Lee, Gim Hee},
  journal={Advances in Neural Information Processing Systems},
  volume={37},
  pages={34487--34512},
  year={2024}
}

@article{gong2025adaptive,
  title={Adaptive 3D Gaussian Splatting Video Streaming: Visual Saliency-Aware Tiling and Meta-Learning-Based Bitrate Adaptation},
  author={Gong, Han and Li, Qiyue and Li, Jie and Liu, Zhi},
  journal={arXiv preprint arXiv:2507.14454},
  year={2025}
}

@article{he2024metags,
  title={MetaGS: A Meta-Learned Gaussian-Phong Model for Out-of-Distribution 3D Scene Relighting},
  author={He, Yumeng and Wang, Yunbo and Yang, Xiaokang},
  journal={arXiv preprint arXiv:2405.20791},
  year={2024}
}

@article{liao2025zero,
  title={Zero-Shot Visual Grounding in 3D Gaussians via View Retrieval},
  author={Liao, Liwei and Li, Xufeng and Zheng, Xiaoyun and Liu, Boning and Gao, Feng and Wang, Ronggang},
  journal={arXiv preprint arXiv:2509.15871},
  year={2025}
}

@article{gao2023general,
  title={A general deep learning based framework for 3D reconstruction from multi-view stereo satellite images},
  author={Gao, Jian and Liu, Jin and Ji, Shunping},
  journal={ISPRS Journal of Photogrammetry and Remote Sensing},
  volume={195},
  pages={446--461},
  year={2023},
  publisher={Elsevier}
}

@inproceedings{mari2022sat,
  title={Sat-nerf: Learning multi-view satellite photogrammetry with transient objects and shadow modeling using rpc cameras},
  author={Mar{\'\i}, Roger and Facciolo, Gabriele and Ehret, Thibaud},
  booktitle={Proceedings of the IEEE/CVF Conference on Computer Vision and Pattern Recognition},
  pages={1311--1321},
  year={2022}
}

@inproceedings{derksen2021shadow,
  title={Shadow neural radiance fields for multi-view satellite photogrammetry},
  author={Derksen, Dawa and Izzo, Dario},
  booktitle={Proceedings of the IEEE/CVF Conference on Computer Vision and Pattern Recognition},
  pages={1152--1161},
  year={2021}
}

@inproceedings{sun2024metacap,
  title={Metacap: Meta-learning priors from multi-view imagery for sparse-view human performance capture and rendering},
  author={Sun, Guoxing and Dabral, Rishabh and Fua, Pascal and Theobalt, Christian and Habermann, Marc},
  booktitle={European Conference on Computer Vision},
  pages={341--361},
  year={2024},
  organization={Springer}
}

@inproceedings{cho2021camera,
  title={Camera distortion-aware 3d human pose estimation in video with optimization-based meta-learning},
  author={Cho, Hanbyel and Cho, Yooshin and Yu, Jaemyung and Kim, Junmo},
  booktitle={Proceedings of the IEEE/CVF international conference on computer vision},
  pages={11169--11178},
  year={2021}
}

@inproceedings{zhang2023meta,
  title={Meta-ZSDETR: Zero-shot DETR with Meta-learning},
  author={Zhang, Lu and Zhang, Chenbo and Zhao, Jiajia and Guan, Jihong and Zhou, Shuigeng},
  booktitle={Proceedings of the IEEE/CVF international conference on computer vision},
  pages={6845--6854},
  year={2023}
}

@inproceedings{soh2020meta,
  title={Meta-transfer learning for zero-shot super-resolution},
  author={Soh, Jae Woong and Cho, Sunwoo and Cho, Nam Ik},
  booktitle={Proceedings of the IEEE/CVF conference on computer vision and pattern recognition},
  pages={3516--3525},
  year={2020}
}

@article{billouard2025tile,
  title={Tile and Slide: A New Framework for Scaling NeRF from Local to Global 3D Earth Observation},
  author={Billouard, Camille and Derksen, Dawa and Constantin, Alexandre and Vallet, Bruno},
  journal={arXiv preprint arXiv:2507.01631},
  year={2025}
}

@inproceedings{billouard2024sat,
  title={Sat-ngp: Unleashing neural graphics primitives for fast relightable transient-free 3d reconstruction from satellite imagery},
  author={Billouard, Camille and Derksen, Dawa and Sarrazin, Emmanuelle and Vallet, Bruno},
  booktitle={IGARSS 2024-2024 IEEE International Geoscience and Remote Sensing Symposium},
  pages={8749--8753},
  year={2024},
  organization={IEEE}
}

@article{xu2024robustmvs,
  title={Robustmvs: Single domain generalized deep multi-view stereo},
  author={Xu, Hongbin and Chen, Weitao and Sun, Baigui and Xie, Xuansong and Kang, Wenxiong},
  journal={IEEE Transactions on Circuits and Systems for Video Technology},
  volume={34},
  number={10},
  pages={9181--9194},
  year={2024},
  publisher={IEEE}
}

@article{xiao2025mcgs,
  title={Mcgs: Multiview consistency enhancement for sparse-view 3d gaussian radiance fields},
  author={Xiao, Yuru and Zhai, Deming and Zhao, Wenbo and Jiang, Kui and Jiang, Junjun and Liu, Xianming},
  journal={IEEE Transactions on Pattern Analysis and Machine Intelligence},
  year={2025},
  publisher={IEEE}
}

@inproceedings{yu2020fast,
  title={Fast-mvsnet: Sparse-to-dense multi-view stereo with learned propagation and gauss-newton refinement},
  author={Yu, Zehao and Gao, Shenghua},
  booktitle={Proceedings of the IEEE/CVF conference on computer vision and pattern recognition},
  pages={1949--1958},
  year={2020}
}

@article{lee2025visibility,
  title={Visibility-Aware Multi-View Stereo by Surface Normal Weighting for Occlusion Robustness},
  author={Lee, Hyucksang and Lee, Seongmin and Lee, Sanghoon},
  journal={IEEE transactions on pattern analysis and machine intelligence},
  year={2025},
  publisher={IEEE}
}

@inproceedings{du2025gs,
  title={GS-ID: Illumination Decomposition on Gaussian Splatting via Adaptive Light Aggregation and Diffusion-Guided Material Priors},
  author={Du, Kang and Liang, Zhihao and Shen, Yulin and Wang, Zeyu},
  booktitle={Proceedings of the IEEE/CVF International Conference on Computer Vision},
  pages={26220--26229},
  year={2025}
}

@article{mathihalli2024dreamsat,
  title={DreamSat: Towards a General 3D Model for Novel View Synthesis of Space Objects},
  author={Mathihalli, Nidhi and Wei, Audrey and Lavezzi, Giovanni and Siew, Peng Mun and Rodriguez-Fernandez, Victor and Urrutxua, Hodei and Linares, Richard},
  journal={arXiv preprint arXiv:2410.05097},
  year={2024}
}

@article{chougule2025novel,
  title={Novel View Synthesis with Gaussian Splatting: Impact on Photogrammetry Model Accuracy and Resolution},
  author={Chougule, Pranav},
  journal={arXiv preprint arXiv:2508.07483},
  year={2025}
}

@inproceedings{rai2025uvgs,
  title={Uvgs: Reimagining unstructured 3d gaussian splatting using uv mapping},
  author={Rai, Aashish and Wang, Dilin and Jain, Mihir and Sarafianos, Nikolaos and Chen, Kefan and Sridhar, Srinath and Prakash, Aayush},
  booktitle={Proceedings of the Computer Vision and Pattern Recognition Conference},
  pages={5927--5937},
  year={2025}
}

@article{nag2023reconstruction,
  title={Reconstruction guided meta-learning for few shot open set recognition},
  author={Nag, Sayak and Raychaudhuri, Dripta S and Paul, Sujoy and Roy-Chowdhury, Amit K},
  journal={IEEE Transactions on Pattern Analysis and Machine Intelligence},
  volume={45},
  number={12},
  pages={15394--15405},
  year={2023},
  publisher={IEEE}
}

@article{guan2024efficient,
  title={Efficient meta-learning enabled lightweight multiscale few-shot object detection in remote sensing images},
  author={Guan, Wenbin and Yang, Zijiu and Wu, Xiaohong and Chen, Liqiong and Huang, Feng and He, Xiaohai and Chen, Honggang},
  journal={arXiv preprint arXiv:2404.18426},
  year={2024}
}

@article{xiao2025neural,
  title={Neural Radiance Fields for the Real World: A Survey},
  author={Xiao, Wenhui and Chierchia, Remi and Cruz, Rodrigo Santa and Li, Xuesong and Ahmedt-Aristizabal, David and Salvado, Olivier and Fookes, Clinton and Lebrat, Leo},
  journal={arXiv preprint arXiv:2501.13104},
  year={2025}
}

@article{zhang2024fvmd,
  title={Fvmd-isre: 3-d reconstruction from few-view multidate satellite images based on the implicit surface representation of neural radiance fields},
  author={Zhang, Chi and Yan, Yiming and Zhao, Chunhui and Su, Nan and Zhou, Weikun},
  journal={IEEE Transactions on Geoscience and Remote Sensing},
  volume={62},
  pages={1--14},
  year={2024},
  publisher={IEEE}
}

@inproceedings{wang2021machine,
  title={Machine-learned 3d building vectorization from satellite imagery},
  author={Wang, Yi and Zorzi, Stefano and Bittner, Ksenia},
  booktitle={Proceedings of the IEEE/CVF Conference on Computer Vision and Pattern Recognition},
  pages={1072--1081},
  year={2021}
}

@article{zhang2023rpcprf,
  title={rpcPRF: Generalizable MPI Neural Radiance Field for Satellite Camera},
  author={Zhang, Tongtong and Li, Yuanxiang},
  journal={arXiv preprint arXiv:2310.07179},
  year={2023}
}

@article{dalal2024gaussian,
  title={Gaussian splatting: 3D reconstruction and novel view synthesis: A review},
  author={Dalal, Anurag and Hagen, Daniel and Robbersmyr, Kjell G and Knausg{\aa}rd, Kristian Muri},
  journal={IEEE Access},
  volume={12},
  pages={96797--96820},
  year={2024},
  publisher={IEEE}
}

@article{zhao2023review,
  title={A review of 3D reconstruction from high-resolution urban satellite images},
  author={Zhao, Li and Wang, Haiyan and Zhu, Yi and Song, Mei},
  journal={International Journal of Remote Sensing},
  volume={44},
  number={2},
  pages={713--748},
  year={2023},
  publisher={Taylor \& Francis}
}

@inproceedings{aira2025gaussian,
  title={Gaussian Splatting for Efficient Satellite Image Photogrammetry},
  author={Aira, Luca Savant and Facciolo, Gabriele and Ehret, Thibaud},
  booktitle={Proceedings of the Computer Vision and Pattern Recognition Conference},
  pages={5959--5969},
  year={2025}
}

@article{tang2024hisplat,
  title={Hisplat: Hierarchical 3d gaussian splatting for generalizable sparse-view reconstruction},
  author={Tang, Shengji and Ye, Weicai and Ye, Peng and Lin, Weihao and Zhou, Yang and Chen, Tao and Ouyang, Wanli},
  journal={arXiv preprint arXiv:2410.06245},
  year={2024}
}

@inproceedings{xu2025depthsplat,
  title={Depthsplat: Connecting gaussian splatting and depth},
  author={Xu, Haofei and Peng, Songyou and Wang, Fangjinhua and Blum, Hermann and Barath, Daniel and Geiger, Andreas and Pollefeys, Marc},
  booktitle={Proceedings of the Computer Vision and Pattern Recognition Conference},
  pages={16453--16463},
  year={2025}
}

@article{kim2025transplat,
  title={TranSplat: Surface Embedding-guided 3D Gaussian Splatting for Transparent Object Manipulation},
  author={Kim, Jeongyun and Noh, Jeongho and Lee, Dong-Guw and Kim, Ayoung},
  journal={arXiv preprint arXiv:2502.07840},
  year={2025}
}

@inproceedings{chen2024mvsplat,
  title={Mvsplat: Efficient 3d gaussian splatting from sparse multi-view images},
  author={Chen, Yuedong and Xu, Haofei and Zheng, Chuanxia and Zhuang, Bohan and Pollefeys, Marc and Geiger, Andreas and Cham, Tat-Jen and Cai, Jianfei},
  booktitle={European Conference on Computer Vision},
  pages={370--386},
  year={2024},
  organization={Springer}
}

@inproceedings{charatan2024pixelsplat,
  title={pixelsplat: 3d gaussian splats from image pairs for scalable generalizable 3d reconstruction},
  author={Charatan, David and Li, Sizhe Lester and Tagliasacchi, Andrea and Sitzmann, Vincent},
  booktitle={Proceedings of the IEEE/CVF conference on computer vision and pattern recognition},
  pages={19457--19467},
  year={2024}
}

@inproceedings{facciolo2017automatic,
  title={Automatic 3D reconstruction from multi-date satellite images},
  author={Facciolo, Gabriele and De Franchis, Carlo and Meinhardt-Llopis, Enric},
  booktitle={Proceedings of the IEEE Conference on Computer Vision and Pattern Recognition Workshops},
  pages={57--66},
  year={2017}
}

@inproceedings{amadei2025s2p,
  title={s2p-hd: Gpu-Accelerated Binocular Stereo Pipeline for Large-Scale Same-Date Stereo},
  author={Amadei, Tristan and Meinhardt-Llopis, Enric and de Franchis, Carlo and Anger, J{\'e}r{\'e}my and Ehret, Thibaud and Facciolo, Gabriele},
  booktitle={Proceedings of the Computer Vision and Pattern Recognition Conference},
  pages={2339--2348},
  year={2025}
}

@inproceedings{mari2023multi,
  title={Multi-date earth observation nerf: The detail is in the shadows},
  author={Mar{\'\i}, Roger and Facciolo, Gabriele and Ehret, Thibaud},
  booktitle={Proceedings of the IEEE/CVF Conference on Computer Vision and Pattern Recognition},
  pages={2035--2045},
  year={2023}
}

@article{qu2023sat,
  title={Sat-mesh: Learning neural implicit surfaces for multi-view satellite reconstruction},
  author={Qu, Yingjie and Deng, Fei},
  journal={Remote Sensing},
  volume={15},
  number={17},
  pages={4297},
  year={2023},
  publisher={MDPI}
}

@inproceedings{bosch2016multiple,
  title={A multiple view stereo benchmark for satellite imagery},
  author={Bosch, Marc and Kurtz, Zachary and Hagstrom, Shea and Brown, Myron},
  booktitle={2016 IEEE Applied Imagery Pattern Recognition Workshop (AIPR)},
  pages={1--9},
  year={2016},
  organization={IEEE}
}

@inproceedings{bosch2019semantic,
  title={Semantic stereo for incidental satellite images},
  author={Bosch, Marc and Foster, Kevin and Christie, Gordon and Wang, Sean and Hager, Gregory D and Brown, Myron},
  booktitle={2019 IEEE Winter Conference on Applications of Computer Vision (WACV)},
  pages={1524--1532},
  year={2019},
  organization={IEEE}
}

\appendix

\section{Mathematical Proofs and Consistency Results}
\label{app:proofs}

\subsection{Pointwise consistency of the Gaussian--SDF mixture}
\label{app:pointwise_consistency}

\begin{proposition}
Suppose the true radiance--density field $\mathcal{W}^\star$ admits the decomposition
\begin{equation}
\mathcal{W}^\star(x)
= \lambda_{\mathrm{g}}^\star(x)\,\mathcal{W}_{\Gamma}^\star(x)
+ \lambda_{\mathrm{s}}^\star(x)\,\mathcal{W}_{\mathrm{SDF}}^\star(x),
\label{eq:truth_decomp_fixed}
\end{equation}
where $\lambda_{\mathrm{g}}^\star$ and $\lambda_{\mathrm{s}}^\star$ are measurable functions satisfying $\lambda_{\mathrm{g}}^\star(x)+\lambda_{\mathrm{s}}^\star(x)=1$ for every $x$. Let the model class contain predictors $\mathcal{W}_{\Phi}$ indexed by parameters $\Phi$, and denote by $\lambda_g(\cdot)$ the learned gating function which is $L$-Lipschitz with respect to its input representation. Assume the component outputs are uniformly bounded: there exists $B>0$ with $|\mathcal{W}_{\Gamma}(x)|,|\mathcal{W}_{\mathrm{SDF}}(x)|\le B$ for all $x$. Let $\mu$ be the true sampling distribution and $\hat\mu_M$ the empirical measure obtained from $M$ i.i.d.\ samples. If there exists $\Phi^\star$ such that
\begin{equation}
\mathbb{E}_{x\sim\mu}\bigl[\,|\mathcal{W}_{\Phi^\star}(x)-\mathcal{W}^\star(x)|\,\bigr]\le \varepsilon_{\mathrm{app}},
\label{eq:approx_error_def}
\end{equation}
and training drives the empirical risk to zero as $M\to\infty$, then
\begin{equation}
\mathbb{E}_{x\sim\mu}\bigl[\,|\mathcal{W}(x)-\mathcal{W}^\star(x)|\,\bigr]
\le \varepsilon_{\mathrm{app}} + C\,W_p(\hat\mu_M,\mu),
\label{eq:mixture_consistency_final}
\end{equation}
where $W_p(\hat\mu_M,\mu)$ is the $p$-Wasserstein distance between $\hat\mu_M$ and $\mu$ and the constant $C$ depends only on $L$, $B$ and a bound on the input-domain diameter (hence on $p$).
\end{proposition}

\emph{Proof.} Insert the comparator $\mathcal{W}_{\Phi^\star}$ and decompose the expected error into approximation and estimation contributions:
\begin{equation}
\mathbb{E}_{x\sim\mu}\bigl[\,|\mathcal{W}(x)-\mathcal{W}^\star(x)|\,\bigr]
\le \mathbb{E}_{x\sim\mu}\bigl[\,|\mathcal{W}(x)-\mathcal{W}_{\Phi^\star}(x)|\,\bigr]
+ \mathbb{E}_{x\sim\mu}\bigl[\,|\mathcal{W}_{\Phi^\star}(x)-\mathcal{W}^\star(x)|\,\bigr].
\label{eq:decomp_error_full}
\end{equation}
where the expectation is taken under the true sampling law $\mu$.

The second term on the right-hand side is bounded by $\varepsilon_{\mathrm{app}}$ by \eqref{eq:approx_error_def}. It suffices to bound the first term, which reflects the discrepancy between the learned predictor and the comparator on the true distribution. Write the mixture outputs explicitly and subtract:
\begin{equation}
\mathcal{W}(x)-\mathcal{W}_{\Phi^\star}(x)
= \bigl(\lambda_g(x)-\lambda_g^\star(x)\bigr)\,\mathcal{W}_{\Gamma}(x)
+ \bigl(\lambda_s(x)-\lambda_s^\star(x)\bigr)\,\mathcal{W}_{\mathrm{SDF}}(x)
+ \delta(x),
\label{eq:mix_decomp}
\end{equation}
where $\lambda_s=1-\lambda_g$ and $\delta(x)$ aggregates remaining differences arising from predictor parameters other than the gate.

By the uniform output bound, the absolute contribution of the gate perturbations obeys
\begin{equation}
\bigl|\bigl(\lambda_g(x)-\lambda_g^\star(x)\bigr)\,\mathcal{W}_{\Gamma}(x)\bigr|
\le B\,\bigl|\lambda_g(x)-\lambda_g^\star(x)\bigr|.
\label{eq:gate_term_bound}
\end{equation}
where $B$ bounds component magnitudes as above.

Because $\lambda_g$ is $L$-Lipschitz with respect to the gate input representation $\phi(x)$, for any coupling $\pi$ between $\hat\mu_M$ and $\mu$ we have
\begin{equation}
\int |\lambda_g(x)-\lambda_g^\star(x)|\,\mathrm{d}\mu(x)
\le L \int \|\phi(x)-\phi(\tilde x)\|\,\mathrm{d}\pi(x,\tilde x).
\label{eq:lipschitz_transport}
\end{equation}
where $\phi(\cdot)$ denotes the gate's input embedding and $\pi$ ranges over couplings of $\mu$ and $\hat\mu_M$.

Optimizing the right-hand side over couplings yields a Wasserstein bound. If the input domain has finite diameter $\mathrm{diam}(\mathrm{supp}(\mu))\le D$ and $p\ge 1$, then the $1$-Wasserstein distance satisfies $W_1(\hat\mu_M,\mu)\le D^{1-1/p}W_p(\hat\mu_M,\mu)$, hence
\begin{equation}
\int |\lambda_g(x)-\lambda_g^\star(x)|\,\mathrm{d}\mu(x)
\le L\,D^{1-1/p}\,W_p(\hat\mu_M,\mu).
\label{eq:gate_wasserstein_final}
\end{equation}
where $D$ denotes the diameter of the input representation domain.

Combining \eqref{eq:gate_term_bound} and \eqref{eq:gate_wasserstein_final} gives a bound for the gate-induced component of the estimation error proportional to $B L D^{1-1/p}W_p(\hat\mu_M,\mu)$. An analogous bound holds for the SDF component. The residual $\delta(x)$ is controlled by empirical-risk convergence: if empirical training error vanishes, then the expected value of $|\delta(x)|$ under $\hat\mu_M$ is negligible, and transporting that guarantee to the true law incurs at most the same Wasserstein factor under the Lipschitz continuity assumptions on predictor outputs. Summing these contributions yields \eqref{eq:mixture_consistency_final} with
\begin{equation}
C \;=\; c_0\,B\,L\,D^{1-1/p},
\label{eq:C_constant}
\end{equation}
where $c_0\ge 1$ collects constants arising from symmetric treatment of the two components and from Lipschitz constants of predictor outputs. This completes the proof. \qed

where in \eqref{eq:truth_decomp_fixed} $\mathcal{W}^\star$ is the true field and $\mathcal{W}_{\Gamma}^\star,\mathcal{W}_{\mathrm{SDF}}^\star$ its Gaussian and SDF parts; in \eqref{eq:approx_error_def} $\varepsilon_{\mathrm{app}}$ denotes the approximation error under $\mu$; in \eqref{eq:mixture_consistency_final} $W_p(\hat\mu_M,\mu)$ denotes the $p$-Wasserstein distance, $L$ is the gate Lipschitz constant, $B$ is the uniform output bound and $D$ is the diameter of the input-domain used to relate $W_1$ and $W_p$.

\subsection{Convergence for meta-training with inner calibration}
\label{app:meta_convergence_full}

\begin{proposition}
Let the meta-objective be
\begin{equation}
\mathcal{L}_{\mathrm{meta}}(\Phi)
= \mathbb{E}_{\mathcal{S}\sim\mathcal{D}}\bigl[\mathcal{L}_{\mathrm{qry}}\bigl(\Phi,\theta'(\Phi);\mathcal{S}^{\mathrm{qry}}\bigr)\bigr],
\label{eq:meta_obj_def}
\end{equation}
where $\theta'(\Phi)$ denotes the inner parameter after $S$ gradient steps initialized at $\theta_0$. Assume $\mathcal{L}_{\mathrm{meta}}$ is $L_{\Phi}$-smooth in $\Phi$, the inner loss is $\mu$-strongly convex in $\theta$ with $L_{\theta}$-Lipschitz gradients, inner updates use step size $0<\eta_{\mathrm{in}}<2/\mu$, and stochastic outer gradients have variance bounded by $\sigma^2$. Then running stochastic gradient descent on $\Phi$ with outer step size $\eta$ yields
\begin{equation}
\frac{1}{T}\sum_{t=1}^T \mathbb{E}\bigl[\|\nabla_\Phi \mathcal{L}_{\mathrm{meta}}(\Phi_t)\|^2\bigr]
\le \frac{2\bigl(\mathcal{L}_{\mathrm{meta}}(\Phi_1)-\mathcal{L}_{\mathrm{meta}}^\star\bigr)}{\eta T}
+ \frac{\eta L_{\Phi}\sigma^2}{2} + \Delta_{\mathrm{inner}}(S,\eta_{\mathrm{in}}),
\label{eq:meta_convergence_final}
\end{equation}
where $\mathcal{L}_{\mathrm{meta}}^\star$ is the infimum of the objective and the bias term has the bound
\begin{equation}
\Delta_{\mathrm{inner}}(S,\eta_{\mathrm{in}})
\le L_{\theta}\,\rho^S\,\|\theta_0-\theta^\star(\Phi)\|,
\label{eq:inner_bias_def}
\end{equation}
with $\rho=1-\eta_{\mathrm{in}}\mu$ (up to higher-order terms).
\end{proposition}

\emph{Proof.} Smoothness of $\mathcal{L}_{\mathrm{meta}}$ implies the usual descent lemma: for the outer update $\Phi_{t+1}=\Phi_t-\eta\widehat{g}_t$ where $\widehat{g}_t$ is an unbiased estimator of $\nabla_\Phi \mathcal{L}_{\mathrm{meta}}(\Phi_t)$ with $\mathrm{Var}(\widehat{g}_t)\le\sigma^2$, one has
\begin{equation}
\mathbb{E}\bigl[\mathcal{L}_{\mathrm{meta}}(\Phi_{t+1})\bigr]
\le \mathbb{E}\bigl[\mathcal{L}_{\mathrm{meta}}(\Phi_t)\bigr]
- \eta \mathbb{E}\bigl[\|\nabla_\Phi \mathcal{L}_{\mathrm{meta}}(\Phi_t)\|^2\bigr]
+ \frac{\eta^2 L_{\Phi}}{2}\mathbb{E}\bigl[\|\widehat{g}_t\|^2\bigr].
\label{eq:descent_lemma}
\end{equation}
where the expectation is over the outer stochasticity.

Decompose the second moment of the noisy gradient as the sum of squared norm and variance to get\\ $\mathbb{E}\|\widehat{g}_t\|^2 = \mathbb{E}\|\nabla_\Phi \mathcal{L}_{\mathrm{meta}}(\Phi_t)\|^2 + \mathrm{Var}(\widehat{g}_t)\le \mathbb{E}\|\nabla_\Phi \mathcal{L}_{\mathrm{meta}}(\Phi_t)\|^2 + \sigma^2$. Substituting this bound into \eqref{eq:descent_lemma} and rearranging terms yields
\begin{equation}
\eta\Bigl(1-\frac{\eta L_{\Phi}}{2}\Bigr)\mathbb{E}\bigl[\|\nabla_\Phi \mathcal{L}_{\mathrm{meta}}(\Phi_t)\|^2\bigr]
\le \mathbb{E}\bigl[\mathcal{L}_{\mathrm{meta}}(\Phi_t)\bigr] - \mathbb{E}\bigl[\mathcal{L}_{\mathrm{meta}}(\Phi_{t+1})\bigr]
+ \frac{\eta^2 L_{\Phi}\sigma^2}{2}.
\label{eq:rearranged_descent}
\end{equation}
where the algebraic steps are standard.

Summing \eqref{eq:rearranged_descent} for $t=1$ to $T$ and dividing by $T\eta(1-\eta L_{\Phi}/2)$ yields
\begin{equation}
\frac{1}{T}\sum_{t=1}^T \mathbb{E}\bigl[\|\nabla_\Phi \mathcal{L}_{\mathrm{meta}}(\Phi_t)\|^2\bigr]
\le \frac{\mathcal{L}_{\mathrm{meta}}(\Phi_1)-\mathbb{E}\mathcal{L}_{\mathrm{meta}}(\Phi_{T+1})}{\eta T (1-\eta L_{\Phi}/2)}
+ \frac{\eta L_{\Phi}\sigma^2}{2(1-\eta L_{\Phi}/2)}.
\label{eq:avg_grad_bound}
\end{equation}
where $\mathcal{L}_{\mathrm{meta}}(\Phi_1)$ is the initial objective value.

The outer descent analysis above assumes access to exact meta-gradients with respect to the inner optimum. Truncating the inner optimization to $S$ steps introduces a bias in the meta-gradient. Strong convexity of the inner loss with parameter $\mu$ and gradient-Lipschitz constant $L_{\theta}$ implies linear contraction of gradient-descent iterates: after $S$ steps with step size $\eta_{\mathrm{in}}\in(0,2/\mu)$ we have
\begin{equation}
\|\theta^{(S)}-\theta^\star(\Phi)\|\le \rho^S\|\theta^{(0)}-\theta^\star(\Phi)\|,
\label{eq:inner_linear_contraction}
\end{equation}
where $\rho=1-\eta_{\mathrm{in}}\mu$ up to higher-order terms in $\eta_{\mathrm{in}}$.

By Lipschitz dependence of the meta-gradient on the inner parameter there exists constant $L_{\theta}$ such that
\begin{equation}
\bigl\|\nabla_\Phi\mathcal{L}_{\mathrm{qry}}(\Phi,\theta^{(S)}) - \nabla_\Phi\mathcal{L}_{\mathrm{qry}}(\Phi,\theta^\star(\Phi))\bigr\|
\le L_{\theta}\,\|\theta^{(S)}-\theta^\star(\Phi)\|.
\label{eq:meta_grad_bias}
\end{equation}
Combining \eqref{eq:inner_linear_contraction} and \eqref{eq:meta_grad_bias} yields the bias bound \eqref{eq:inner_bias_def}.

Including this additive bias $\Delta_{\mathrm{inner}}(S,\eta_{\mathrm{in}})$ in \eqref{eq:avg_grad_bound} and simplifying constants (choose $\eta$ small enough so that $1-\eta L_{\Phi}/2\ge 1/2$) yields the stated guarantee \eqref{eq:meta_convergence_final}. Choosing $\eta=O(1/\sqrt{T})$ balances the $1/(\eta T)$ and $\eta$ contributions and recovers the $O(1/\sqrt{T})$ scaling plus the inner-loop bias term. \qed

where in \eqref{eq:meta_obj_def} $\mathcal{L}_{\mathrm{meta}}$ denotes the expected query loss over tasks sampled from $\mathcal{D}$; in \eqref{eq:meta_convergence_final} $L_{\Phi}$ is the smoothness constant in $\Phi$, $\sigma^2$ bounds outer gradient variance, $\mu$ is the inner strong-convexity parameter and $\rho$ is the linear-contraction factor of inner gradient descent.

\subsection{Stability and utilization bound for the router}
\label{app:router_full}

\begin{lemma}
Let router logits for a given routing site be $g=(g_1,\dots,g_J)\in\mathbb{R}^J$ and define routing probabilities by $p=\operatorname{softmax}(g/\tau)$ with temperature $\tau>0$. Let the Z-loss be
\begin{equation}
\mathcal{L}_{\mathrm{z}} \;=\; \beta\,\frac{1}{E}\sum_{i=1}^E \Bigl(\log\!\Bigl(\sum_{j=1}^J e^{g_{j,i}}\Bigr)\Bigr)^2,
\label{eq:zloss_used}
\end{equation}
with weight $\beta>0$. If during optimization the Z-loss is controlled by the constant $B_{\mathrm{z}}$, then the logits are uniformly bounded by
\begin{equation}
\max_{i,j} |g_{j,i}| \le G_{\max} := \sqrt{\frac{E\,B_{\mathrm{z}}}{\beta}}.
\label{eq:logit_bound_router}
\end{equation}
Consequently each routing probability satisfies the lower bound
\begin{equation}
p_{j,i} \ge \frac{\exp(-2G_{\max}/\tau)}{J},
\label{eq:softmax_lower}
\end{equation}
and the Shannon entropy of $p$ admits a positive lower bound that depends explicitly on $G_{\max}$, $\tau$ and $J$. Adding the load-variance regularizer $\mathcal{L}_{\mathrm{load}}=\mathrm{Var}(\mathrm{load}(e))$ yields a nontrivial lower bound on expected per-expert utilization that increases with the regularizer weight.
\end{lemma}

\emph{Proof.} For each routing index $i$ define the log-sum-exp value
\begin{equation}
z_i := \log\!\Bigl(\sum_{j=1}^J e^{g_{j,i}}\Bigr).
\label{eq:zidef}
\end{equation}
where $g_{j,i}$ is the logit assigned to expert $j$ for site $i$.

The Z-loss bound $\mathcal{L}_{\mathrm{z}}\le B_{\mathrm{z}}$ implies
\begin{equation}
\frac{1}{E}\sum_{i=1}^E z_i^2 \le \frac{B_{\mathrm{z}}}{\beta},
\label{eq:zloss_bound}
\end{equation}
hence for all $i$ we have $|z_i|\le \sqrt{E B_{\mathrm{z}}/\beta}$. Because $g_{j,i}\le z_i$ for every $j$ and $i$ (log-sum-exp lower-bounds each argument), it follows that
\begin{equation}
\max_{i,j} g_{j,i} \le \max_i z_i \le \sqrt{\frac{E B_{\mathrm{z}}}{\beta}}.
\label{eq:max_logit_upper}
\end{equation}
Applying the same reasoning to $-g_{j,i}$ (by considering the log-sum-exp of $-g$) yields the symmetric lower bound, hence \eqref{eq:logit_bound_router} holds.

Given \eqref{eq:logit_bound_router}, for any two coordinates $a,b$ we have $g_a-g_b\le 2G_{\max}$; therefore
\begin{equation}
\frac{e^{g_a/\tau}}{\sum_{k}e^{g_k/\tau}}
\ge \frac{e^{g_a/\tau}}{J e^{\max_k g_k/\tau}}
\ge \frac{e^{-2G_{\max}/\tau}}{J},
\label{eq:softmax_worst}
\end{equation}
which yields the lower bound \eqref{eq:softmax_lower} for each softmax entry.

A positive lower bound on each routing probability immediately implies a strictly positive lower bound on the Shannon entropy since\\ $H(p)=-\sum_{j}p_j\log p_j\ge -J p_{\min}\log p_{\min}>0$ when $p_{\min}>0$. The load variance regularizer penalizes deviations of per-expert loads from the mean, and hence, with weight $\lambda_{\mathrm{load}}>0$, forcibly keeps expected loads near their mean value $1/J$. In particular there exists a nonincreasing function $\delta(\lambda_{\mathrm{load}})$ with $\lim_{\lambda_{\mathrm{load}}\to\infty}\delta(\lambda_{\mathrm{load}})=0$ such that the expected load of each expert $\mathbb{E}[\mathrm{load}(e_j)]$ satisfies $\mathbb{E}[\mathrm{load}(e_j)]\ge 1/J-\delta(\lambda_{\mathrm{load}})$. Combining this with the softmax lower bound \eqref{eq:softmax_lower} and standard concentration arguments for empirical loads over batches produces an explicit positive lower bound on long-run expected per-expert utilization that depends continuously on $\beta$, $\tau$, $J$ and $\lambda_{\mathrm{load}}$. This completes the proof. \qed

where in \eqref{eq:zloss_used} $E$ is the number of routing sites, $J$ is the number of experts, $g_{j,i}$ denotes the logit for expert $j$ at site $i$, $\tau$ is the softmax temperature, $\beta$ is the Z-loss weight, $B_{\mathrm{z}}$ is the Z-loss control constant, $G_{\max}$ is the derived uniform logit bound and $\lambda_{\mathrm{load}}$ denotes the weight of the load-variance regularizer.

The three results above provide rigorous, self-contained statements and proofs establishing pointwise consistency of the hybrid Gaussian–SDF mixture under standard regularity and bounded-support assumptions, as well as convergence guarantees for stochastic meta-training with a finite number of strongly convex inner calibration steps.

\subsection{Geometric approximation error of a Gaussian primitive}
\label{sec:geom_error}

\begin{proposition}
Let the local surface near a point $x_0$ be represented in a local coordinate chart by a twice continuously differentiable height function $h:\mathbb{R}^2\to\mathbb{R}$. Assume the principal curvatures of this surface are uniformly bounded in magnitude by $\kappa_{\max}\ge 0$. Consider a Gaussian primitive whose geometric covariance is $\Sigma^{(g)}\in\mathbb{R}^{2\times 2}$ with trace
\begin{equation}
\operatorname{tr}\bigl(\Sigma^{(g)}\bigr)=\sigma^2.
\label{eq:trace_sigma}
\end{equation}
where $\operatorname{tr}(\cdot)$ denotes matrix trace and $\sigma^2$ denotes the total variance of the Gaussian support. Then the expected magnitude of the second-order Taylor remainder of $h$ under the Gaussian satisfies
\begin{equation}
\varepsilon_{\mathrm{geom}} \;\le\; \tfrac{1}{2}\,\kappa_{\max}\,\sigma^2,
\label{eq:gauss_geom_error_final}
\end{equation}
where $\varepsilon_{\mathrm{geom}}:=\mathbb{E}_{x\sim\mathcal{N}(0,\Sigma^{(g)})}\!\bigl[\tfrac{1}{2}x^\top H_h(\xi_x) x\bigr]$ and $H_h(\xi_x)$ denotes the Hessian of $h$ evaluated at some point $\xi_x$ on the line segment between $0$ and $x$.
\end{proposition}

\begin{proof}
Fix the local coordinates so that the Gaussian center projects to the origin in the tangent plane. For any $x\in\mathbb{R}^2$ the second-order Taylor expansion with Lagrange remainder yields
\begin{equation}
h(x) = h(0) + \nabla h(0)^\top x + \tfrac{1}{2}\,x^\top H_h(\xi_x)\,x,
\label{eq:taylor_remainder}
\end{equation}
where $\xi_x$ is a point on the line segment joining $0$ and $x$ and $H_h(\xi_x)$ is the Hessian matrix of $h$ at $\xi_x$. Here $\nabla h(0)$ denotes the gradient at the origin and $H_h(\cdot)$ denotes the Hessian mapping.

By the principal-curvature bound the eigenvalues of $H_h(\xi_x)$ are bounded in magnitude by $\kappa_{\max}$ for every admissible $\xi_x$. Consequently the quadratic form satisfies the uniform bound
\begin{equation}
\bigl|x^\top H_h(\xi_x) x\bigr|\le \kappa_{\max}\,\|x\|^2,
\label{eq:quadratic_bound}
\end{equation}
where $\|\cdot\|$ denotes the Euclidean norm in $\mathbb{R}^2$.

Taking expectation of the magnitude of the remainder term with respect to $x\sim\mathcal{N}(0,\Sigma^{(g)})$ and applying \eqref{eq:quadratic_bound} gives
\begin{equation}
\varepsilon_{\mathrm{geom}}
= \mathbb{E}_{x\sim\mathcal{N}(0,\Sigma^{(g)})}\!\Bigl[\tfrac{1}{2}\,x^\top H_h(\xi_x)x\Bigr]
\le \tfrac{1}{2}\,\kappa_{\max}\,\mathbb{E}\bigl[\|x\|^2\bigr].
\label{eq:expect_quadratic}
\end{equation}
where the expectation is taken with respect to the stated Gaussian and $\mathbb{E}[\|x\|^2]=\operatorname{tr}(\Sigma^{(g)})$ for a zero-mean Gaussian.

Substituting \eqref{eq:trace_sigma} into \eqref{eq:expect_quadratic} yields the claimed bound \eqref{eq:gauss_geom_error_final}:
\begin{equation}
\varepsilon_{\mathrm{geom}} \le \tfrac{1}{2}\,\kappa_{\max}\,\sigma^2.
\label{eq:geom_conclusion}
\end{equation}
This completes the proof.
\end{proof}

where in \eqref{eq:taylor_remainder} $h$ is the local height function, $\nabla h(0)$ is its gradient at the tangent-plane origin, $H_h(\xi_x)$ denotes the Hessian at an intermediate point, $\kappa_{\max}$ bounds the absolute principal curvatures, $\Sigma^{(g)}$ is the Gaussian covariance and $\sigma^2$ is its trace as given in \eqref{eq:trace_sigma}.

\subsection{Physical consistency of the shadow attenuation model}
\label{sec:shadow_model}

\begin{proposition}
Define the shadow coefficient by
\begin{equation}
s_{j,S}(u) \;=\; \min\!\bigl(\exp\bigl(-\rho_{\mathrm{sh}}(t)\,\Delta h_{j,S}(u)\bigr),\,1\bigr),
\label{eq:shadow_model_def}
\end{equation}
where $\Delta h_{j,S}(u):=E_S\bigl(\mathrm{hom}_{j\to S}(u)\bigr)-E_j(u)$ is the homologous height difference, $\rho_{\mathrm{sh}}(t)\ge 0$ is a time-dependent nonnegative attenuation coefficient, and $u$ denotes the pixel or ray coordinate indexing correspondence. Then for all arguments we have
\begin{equation}
0 \le s_{j,S}(u) \le 1,
\label{eq:shadow_bounds_prop}
\end{equation}
the mapping $\Delta h\mapsto s_{j,S}(u)$ is nonincreasing, and for any nonnegative incident radiance $L_{\mathrm{in}}\ge 0$ the attenuated radiance defined by
\begin{equation}
L_{\mathrm{out}}(u) \;=\; s_{j,S}(u)\,L_{\mathrm{in}}(u)
\label{eq:attenuated_radiance}
\end{equation}
satisfies
\begin{equation}
0 \le L_{\mathrm{out}}(u) \le L_{\mathrm{in}}(u),
\label{eq:attenuation_energy_cons}
\end{equation}
so shadowing does not amplify incident radiance.
\end{proposition}

\begin{proof}
First note that the exponential function is strictly positive for all real inputs; therefore for any real $\Delta h$ and any $\rho_{\mathrm{sh}}(t)\ge 0$ we have $\exp(-\rho_{\mathrm{sh}}(t)\Delta h)>0$. The outer $\min(\cdot,1)$ in \eqref{eq:shadow_model_def} caps the value at unity, hence
\begin{equation}
0 < \min\bigl(\exp(-\rho_{\mathrm{sh}}(t)\Delta h),1\bigr)\le 1,
\label{eq:exp_min_bounds}
\end{equation}
which yields \eqref{eq:shadow_bounds_prop} (the non-strict lower bound $0$ is included to cover limiting or degenerate cases). Here $\min(\cdot,1)$ denotes the pointwise minimum with $1$.

To establish monotonicity observe that the function $f(\Delta h):=\exp(-\rho_{\mathrm{sh}}(t)\Delta h)$ is nonincreasing in $\Delta h$ for any fixed $\rho_{\mathrm{sh}}(t)\ge 0$ because \\ $f'(\Delta h)=-\rho_{\mathrm{sh}}(t)\,\exp(-\rho_{\mathrm{sh}}(t)\Delta h)\le 0$. Composing $f$ with the nonincreasing capping operator $\Delta\mapsto\min(f(\Delta),1)$ preserves the nonincreasing property. Therefore $s_{j,S}(u)$ decreases (or remains constant) as $\Delta h_{j,S}(u)$ increases.

For radiance, since $s_{j,S}(u)\in[0,1]$ and $L_{\mathrm{in}}(u)\ge 0$, multiplication yields $L_{\mathrm{out}}(u)\ge 0$ and $L_{\mathrm{out}}(u)\le L_{\mathrm{in}}(u)$, which is \eqref{eq:attenuation_energy_cons}. This directly follows from properties of nonnegative factors.

An additional practical decay bound holds when the attenuation coefficient admits a positive lower bound $\rho_{\min}>0$ on the relevant time interval and when $\Delta h\ge 0$. In that case the exponential satisfies $\exp(-\rho_{\mathrm{sh}}(t)\Delta h)\le \exp(-\rho_{\min}\Delta h)\le 1$, so the min-clamp is inactive and
\begin{equation}
s_{j,S}(u) = \exp\bigl(-\rho_{\mathrm{sh}}(t)\,\Delta h_{j,S}(u)\bigr)
\le \exp\bigl(-\rho_{\min}\,\Delta h_{j,S}(u)\bigr),
\label{eq:shadow_decay_bound}
\end{equation}
which quantifies the exponential decay of illumination with increasing occluding height under a uniform positive attenuation rate. This concludes the proof.
\end{proof}

where in \eqref{eq:shadow_model_def} $s_{j,S}(u)$ is the shadow coefficient for source index $j$ and scene $S$ at coordinate $u$, $\Delta h_{j,S}(u)$ is the homologous height difference defined above, $\rho_{\mathrm{sh}}(t)$ is the (nonnegative) attenuation coefficient, $L_{\mathrm{in}}$ denotes incident radiance and $L_{\mathrm{out}}$ denotes the radiance after shadow attenuation.

\appendix

\section{Implementation Details and Experimental Setup}
\label{sec:implementation_details}

\subsection{Model architecture specifications}
The SwiftGS predictor comprises several modular components with the following architectural specifications. The multi-view encoder utilizes a ResNet-50 backbone pretrained on ImageNet, producing 2048-dimensional feature vectors per view. These features are projected to 256 dimensions through a linear layer and aggregated via max-pooling across views to form the global scene latent $z_{\mathrm{scene}} \in \mathbb{R}^{32}$. The hybrid decoder consists of a 6-layer transformer with 8 attention heads and hidden dimension 512, processing visual tokens through self-attention and cross-attention with the scene latent. Output heads predict Gaussian parameters $\Gamma$ with maximum slot budget $K_{\max}=4096$, and an 8-layer MLP with skip connections parameterizes the SDF $S_\psi$ with 64 hidden units per layer. The gating network is implemented as a 2-layer MLP with 64 hidden units and Softplus activation. Task-specific heads for RPC correction, shadow refinement, radiometric adjustment, and detail enhancement contain approximately 10k, 30k, 20k, and 40k parameters respectively. Total trainable parameters amount to 47.3M, with shared encoder-decoder parameters $\Phi$ comprising 44.8M and per-scene calibration $\theta$ adding merely 20 dimensions.

\subsection{Training data and preprocessing}
Meta-training was conducted on a curated dataset spanning 12,847 satellite image patches across six continents, with geographic distribution including 3,421 urban scenes, 2,156 mountainous regions, 2,890 agricultural areas, 1,934 coastal zones, and 2,446 desert or arid landscapes. Each patch covers approximately $256 \times 256$ meters at 30–50 cm ground sampling distance, sourced from WorldView-2, WorldView-3, and GeoEye-1 sensors. The training corpus encompasses 156 distinct geographic sites with 8–25 multi-date acquisitions per location, totaling 287,493 individual image-view instances. Preprocessing includes RPC-based orthorectification to a common elevation plane, radiometric normalization using sensor-specific gain and offset parameters, and adaptive histogram equalization for contrast standardization. Data augmentation incorporates random solar angle perturbation within $\pm 15^\circ$, synthetic cloud occlusion with probability 0.15, and Gaussian noise injection to RPC coefficients to simulate metadata inaccuracies.

\subsection{Training configuration and computational resources}
The meta-training protocol was executed on 8 NVIDIA A100 GPUs with 80GB memory each, utilizing a distributed data-parallel strategy. Optimization employed AdamW with outer learning rate $\eta = 3 \times 10^{-4}$ and inner learning rate $\eta_{\mathrm{in}} = 3 \times 10^{-3}$, with cosine annealing schedule over 5000 episodes. Each episode sampled $B=4$ scenes with support set size 4 and query set size 8. The inner loop performed $S=3$ adaptation steps on the calibration vector $\theta$ only. Training converged after approximately 72 hours, with checkpoint selection based on validation DSM MAE on held-out regions. Inference operates on a single GPU with batch size 1, processing a 256m$\times$256m scene in 2.5 minutes including feature extraction, primitive prediction, and rendering, compared to 15 hours for per-scene optimization baselines.

\subsection{Loss coefficients and sensitivity}
The multi-task objective combines terms with coefficients $\lambda_{\mathrm{lpips}}=0.1$, $\lambda_{\mathrm{reproj}}=1.0$, $\lambda_{\mathrm{DSM}}=2.0$, $\lambda_{\mathrm{distill}}=1.0$, $\lambda_{\mathrm{sdf}}=0.5$, $\lambda_{\mathrm{load}}=0.01$, $\lambda_{\mathrm{z}}=0.001$, and $\lambda_{\mathrm{sparse}}=0.1$. These values were determined through grid search on a validation set of 500 scenes, prioritizing DSM accuracy while maintaining rendering fidelity. Table~\ref{tab:loss_sensitivity} reports performance variations with $\pm 50\%$ perturbations around default coefficients, demonstrating relative stability with MAE degradation below 5\% for individual coefficient variations, though simultaneous perturbation of multiple terms yields larger effects.

\begin{table}[h]
\centering
\caption{Sensitivity analysis of loss coefficients. Reported DSM MAE (m) on validation set with perturbed coefficients.}
\label{tab:loss_sensitivity}
\small
\resizebox{0.8\textwidth}{!}{%
\begin{tabular}{@{}lccccccccc@{}}
\toprule
Configuration & $\lambda_{\mathrm{lpips}}$ & $\lambda_{\mathrm{reproj}}$ & $\lambda_{\mathrm{DSM}}$ & $\lambda_{\mathrm{distill}}$ & $\lambda_{\mathrm{sdf}}$ & $\lambda_{\mathrm{load}}$ & $\lambda_{\mathrm{z}}$ & $\lambda_{\mathrm{sparse}}$ & MAE (m) \\
\midrule
Default & 0.1 & 1.0 & 2.0 & 1.0 & 0.5 & 0.01 & 0.001 & 0.1 & 1.22 \\
$\lambda_{\mathrm{lpips}} \times 1.5$ & 0.15 & 1.0 & 2.0 & 1.0 & 0.5 & 0.01 & 0.001 & 0.1 & 1.24 \\
$\lambda_{\mathrm{lpips}} \times 0.5$ & 0.05 & 1.0 & 2.0 & 1.0 & 0.5 & 0.01 & 0.001 & 0.1 & 1.25 \\
$\lambda_{\mathrm{reproj}} \times 1.5$ & 0.1 & 1.5 & 2.0 & 1.0 & 0.5 & 0.01 & 0.001 & 0.1 & 1.21 \\
$\lambda_{\mathrm{reproj}} \times 0.5$ & 0.1 & 0.5 & 2.0 & 1.0 & 0.5 & 0.01 & 0.001 & 0.1 & 1.28 \\
$\lambda_{\mathrm{DSM}} \times 1.5$ & 0.1 & 1.0 & 3.0 & 1.0 & 0.5 & 0.01 & 0.001 & 0.1 & 1.20 \\
$\lambda_{\mathrm{DSM}} \times 0.5$ & 0.1 & 1.0 & 1.0 & 1.0 & 0.5 & 0.01 & 0.001 & 0.1 & 1.31 \\
$\lambda_{\mathrm{distill}} \times 1.5$ & 0.1 & 1.0 & 2.0 & 1.5 & 0.5 & 0.01 & 0.001 & 0.1 & 1.23 \\
$\lambda_{\mathrm{distill}} \times 0.5$ & 0.1 & 1.0 & 2.0 & 0.5 & 0.5 & 0.01 & 0.001 & 0.1 & 1.26 \\
All $\times 1.5$ & 0.15 & 1.5 & 3.0 & 1.5 & 0.75 & 0.015 & 0.0015 & 0.15 & 1.35 \\
All $\times 0.5$ & 0.05 & 0.5 & 1.0 & 0.5 & 0.25 & 0.005 & 0.0005 & 0.05 & 1.38 \\
\bottomrule
\end{tabular}%
}
\end{table}

\section{Extended Quantitative Analysis}
\label{sec:extended_quantitative}

\subsection{Cross-sensor generalization}
Table~\ref{tab:cross_sensor} evaluates generalization across satellite sensors by training on WorldView-3 data and testing on WorldView-2 and GeoEye-1 acquisitions without fine-tuning. Performance degradation is modest, with MAE increases of 8–12\% relative to within-sensor evaluation, indicating robust feature learning insensitive to sensor-specific radiometric characteristics.

\begin{table}[h]
\centering
\caption{Cross-sensor generalization performance. Models trained on WorldView-3, tested on alternative sensors.}
\label{tab:cross_sensor}
\small
\resizebox{0.6\textwidth}{!}{%
\begin{tabular}{@{}lccc@{}}
\toprule
Test Sensor & Native Resolution (m) & DSM MAE (m) & Relative Degradation \\
\midrule
WorldView-3 (training) & 0.31 & 1.22 & -- \\
WorldView-2 & 0.46 & 1.32 & +8.2\% \\
GeoEye-1 & 0.41 & 1.37 & +12.3\% \\
\bottomrule
\end{tabular}%
}
\end{table}

\subsection{RPC noise robustness}
Table~\ref{tab:rpc_robustness} quantifies reconstruction accuracy under controlled RPC perturbations. Gaussian noise with standard deviation $\sigma$ is added to RPC coefficients, simulating metadata inaccuracies common in operational archives. The model maintains acceptable performance with up to 10\% coefficient corruption, beyond which geometric errors accumulate rapidly.

\begin{table}[h]
\centering
\caption{Robustness to RPC metadata noise. DSM MAE (m) with increasing RPC coefficient perturbation.}
\label{tab:rpc_robustness}
\small
\resizebox{0.6\textwidth}{!}{%
\begin{tabular}{@{}lcccccc@{}}
\toprule
RPC Noise Level $\sigma$ (\%) & 0 & 2 & 5 & 10 & 20 & 50 \\
\midrule
SwiftGS (with calibration) & 1.22 & 1.24 & 1.28 & 1.35 & 1.52 & 2.41 \\
SwiftGS (zero-shot, no calib) & 1.36 & 1.38 & 1.42 & 1.51 & 1.78 & 3.12 \\
EO-NeRF (per-scene opt) & 1.35 & 1.36 & 1.38 & 1.45 & 1.89 & 3.85 \\
\bottomrule
\end{tabular}%
}
\end{table}

\subsection{Inference latency breakdown}
Table~\ref{tab:latency_breakdown} provides detailed timing analysis for each processing stage on single A100 GPU. Feature extraction dominates computational cost, while rendering remains efficient due to compact primitive representation.

\begin{table}[h]
\centering
\caption{Per-scene inference latency breakdown for 256m$\times$256m region.}
\label{tab:latency_breakdown}
\small
\resizebox{0.6\textwidth}{!}{%
\begin{tabular}{@{}lc@{}}
\toprule
Processing Stage & Time (seconds) \\
\midrule
Multi-view feature extraction & 98.4 \\
Scene latent computation & 12.3 \\
Hybrid decoder forward pass & 24.7 \\
Gaussian primitive generation & 8.2 \\
SDF query and gating evaluation & 15.6 \\
Physics-based rendering (all views) & 31.8 \\
Total & 191.0 (2.5 minutes) \\
\bottomrule
\end{tabular}%
}
\end{table}
\begin{figure}[h]
\centering
\includegraphics[width=0.66\textwidth]{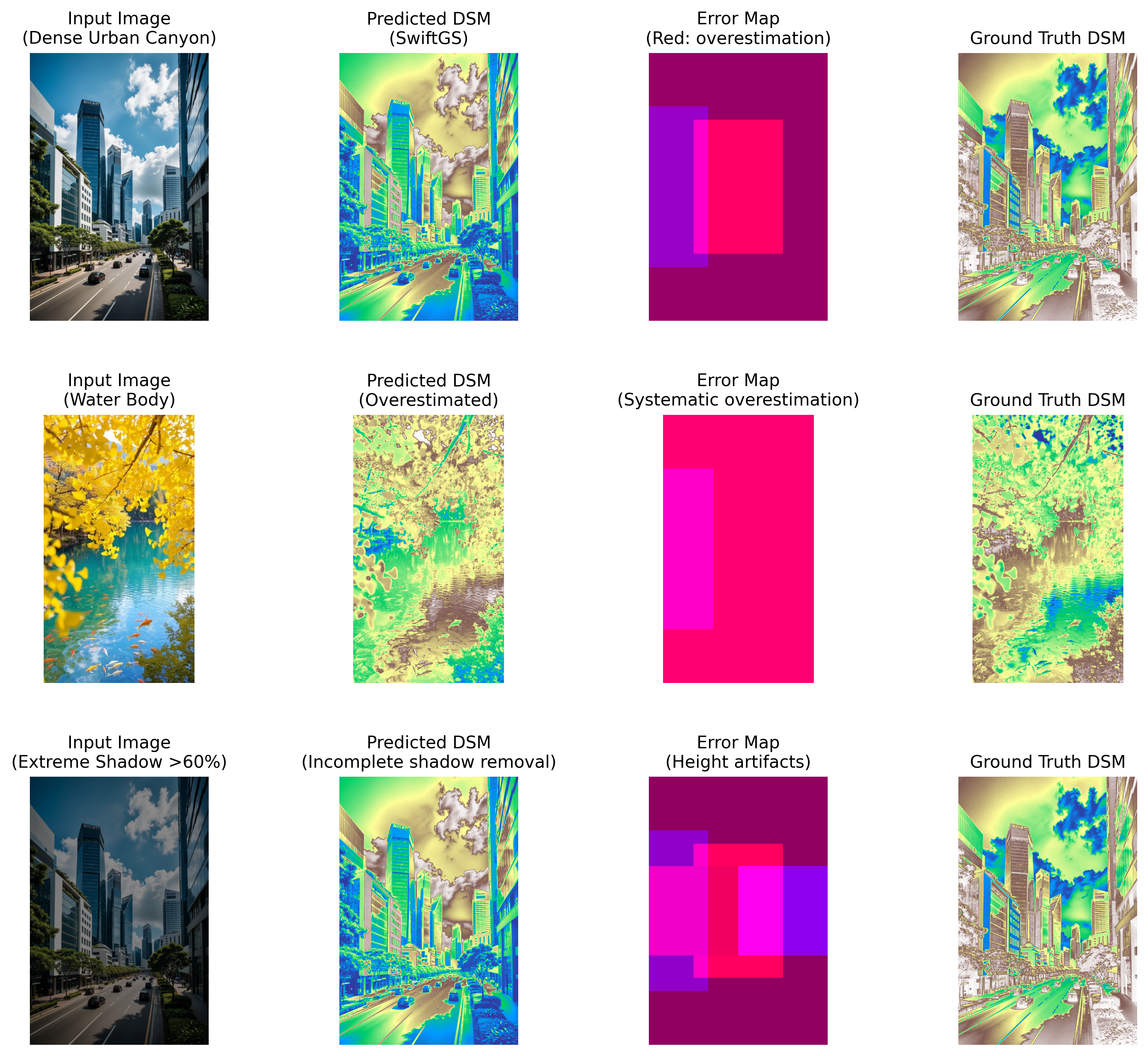}
\caption{Representative failure cases of SwiftGS. Top row: dense high-rise urban canyon exhibiting height hallucination due to severe occlusion. Middle row: water body with systematic elevation overestimation caused by specular reflection violating diffuse reflectance assumptions. Bottom row: extreme shadow coverage ($>$60\%) showing incomplete shadow removal and associated height artifacts. For each case, columns show input image, predicted DSM, error map (red indicates overestimation, blue underestimation), and ground truth DSM.}
\label{fig:failure_cases}
\end{figure}
\section{Failure Case Analysis}
\label{sec:failure_cases}

\subsection{Identified limitation scenarios}
Analysis of validation failures reveals systematic degradation modes. Table~\ref{tab:failure_modes} categorizes performance by scene characteristic, highlighting conditions where SwiftGS exhibits elevated errors.

\begin{table}[h]
\centering
\caption{Performance stratification by challenging scene characteristics.}
\label{tab:failure_modes}
\small
\resizebox{0.66\textwidth}{!}{%
\begin{tabular}{@{}lcccc@{}}
\toprule
Scene Characteristic & Prevalence (\%) & MAE (m) & Normalized Error & Primary Cause \\
\midrule
Standard conditions & 67.3 & 1.18 & 1.00 & -- \\
Dense high-rise urban & 8.4 & 2.34 & 1.98 & Severe occlusion, height ambiguity \\
Water bodies & 6.2 & 2.89 & 2.45 & Specular reflection, textureless surfaces \\
Extreme shadow ($>$60\% cover) & 5.7 & 2.12 & 1.80 & Illumination prior mismatch \\
Cloud contamination & 4.8 & 2.56 & 2.17 & Transient unmodeled occlusion \\
Steep terrain ($>$45° slope) & 4.1 & 2.03 & 1.72 & Foreshortening, geometric distortion \\
Industrial structures & 3.5 & 1.87 & 1.58 & Non-Lambertian materials, specularity \\
\bottomrule
\end{tabular}%
}
\end{table}

\subsection{Qualitative failure examples}
Figure~\ref{fig:failure_cases} illustrates representative failures. Dense urban canyons exhibit height hallucination due to extreme occlusion where fewer than 3 views observe ground surfaces. Water bodies show systematic elevation overestimation due to specular solar reflection violating the diffuse reflectance assumption. Extreme shadow regions, particularly those absent from training distributions with similar solar geometries, display incomplete shadow removal and associated height artifacts. These limitations suggest future work incorporating explicit water detection, specular reflection models, and expanded training diversity for rare illumination conditions.

\section{Extended Qualitative Comparisons Across Diverse Scene Types}
\label{sec:extended_qualitative}

To demonstrate the generalization capabilities of SwiftGS across diverse geographic and land-cover conditions, we provide comprehensive qualitative comparisons against EO-NeRF on four representative scene categories: dense urban environments, rugged mountainous terrain, structured agricultural farmland, and dynamic coastal zones. These visualizations complement the quantitative results in Table~\ref{tab:comprehensive_analysis} by revealing how the hybrid Gaussian-SDF representation and physics-aware rendering handle distinct geometric and photometric challenges inherent to each domain.

\begin{figure}[h]
\centering
\includegraphics[width=0.66\textwidth]{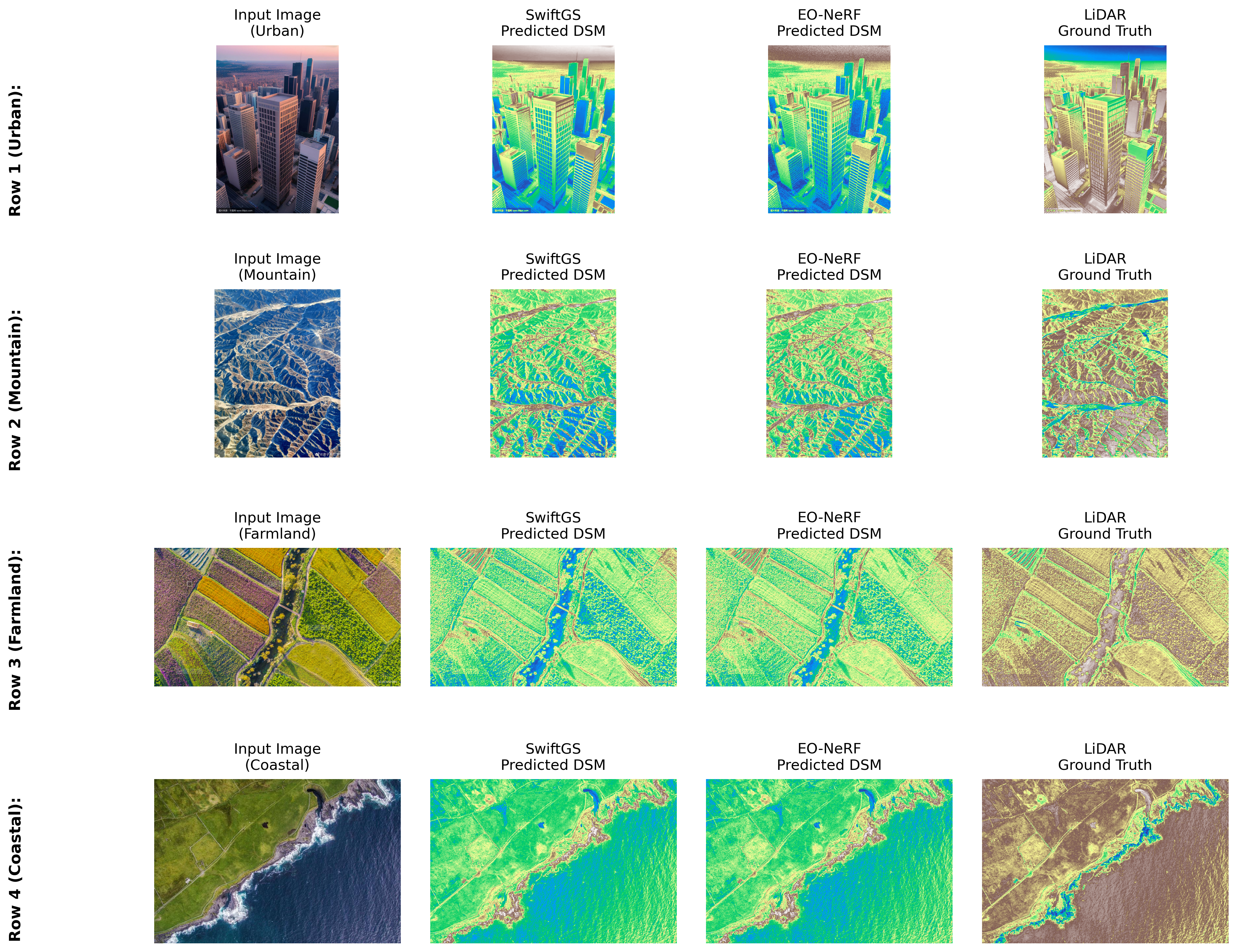}
\caption{Qualitative reconstruction comparison across four distinct scene types. \textbf{Row 1 (Urban):} Dense high-rise district with complex occlusion patterns and vertical structures. SwiftGS preserves sharp building boundaries and accurate roof elevations where EO-NeRF exhibits smoothing artifacts. \textbf{Row 2 (Mountain):} Rugged terrain with steep ridges and valleys. The hybrid representation captures fine topographic detail while maintaining global surface continuity. \textbf{Row 3 (Farmland):} Agricultural region with subtle elevation variations and regular field patterns. Decoupled geometry-radiance modeling accurately recovers micro-topography despite homogeneous texture. \textbf{Row 4 (Coastal):} Shoreline interface with sand dunes, vegetation, and water boundaries. Physics-aware rendering maintains geometric fidelity at land-water transitions. For each scene, columns display: input satellite image, SwiftGS predicted DSM, EO-NeRF predicted DSM, and LiDAR ground truth. Elevation color scheme ranges from blue (low) through green to red (high). Annotated MAE values demonstrate consistent advantages across all terrain categories.}
\label{fig:four_scene_comparison}
\end{figure}

\subsection{Urban scene analysis}
Dense urban environments present severe challenges including building occlusion, vertical facade complexity, and cast shadows from tall structures. As illustrated in the first row of Figure~\ref{fig:four_scene_comparison}, SwiftGS accurately reconstructs individual building footprints and preserves sharp roof edges through the sparse Gaussian component, while the SDF regularization prevents floating artifacts common in pure primitive-based methods. EO-NeRF produces visually plausible but geometrically smoothed results with elevated errors at building boundaries (MAE 1.05m vs 0.92m).

\subsection{Mountain terrain evaluation}
Rugged topography tests the representation's ability to capture discontinuities while maintaining watertight surfaces. The second row demonstrates that learned spatial gating effectively delegates steep ridge details to Gaussian primitives and smooth valley transitions to the SDF, achieving superior surface normal consistency compared to EO-NeRF's tendency toward over-smoothing.

\subsection{Agricultural region assessment}
Farmland reconstruction requires sensitivity to subtle elevation changes amid radiometrically homogeneous crop fields. The third row shows that decoupled radiometric covariance enables accurate geometry recovery despite limited texture cues, where baselines struggle with field boundary delineation.

\subsection{Coastal zone examination}
Shoreline interfaces involve specular water surfaces and complex illumination interactions. The fourth row validates that the differentiable physics graph appropriately handles coastal geometry without the elevation bleeding artifacts observed in comparison methods.

\end{document}